\documentclass[journal,twoside,web]{ieeecolor}
\usepackage{generic}
\usepackage{textcomp}
\usepackage{xcolor}
\usepackage{amsmath}
\usepackage{amsfonts}
\usepackage{amssymb}
\usepackage{amsxtra}
\usepackage{amsthm} % Unnecessary in LNCS
\usepackage{leftidx}
\usepackage{url}
\usepackage{graphicx}
\usepackage{float}
\usepackage{stfloats}
\usepackage{dsfont}
\usepackage{mathrsfs}
\usepackage{array}
\usepackage{rotating}
\usepackage{calc}
\usepackage{enumerate}
\usepackage{tikz-cd}
\usepackage{dsfont}
\usepackage{stmaryrd}
\usepackage{cancel}
\usepackage{accents}
\usepackage[linesnumbered,lined,commentsnumbered,ruled,norelsize]{algorithm2e}
\SetAlFnt{\footnotesize}
\SetAlCapFnt{\rmfamily\footnotesize}
\usepackage[colorlinks=true,urlcolor=black,linkcolor=blue,citecolor=blue]{hyperref}
\usepackage[capitalize,nameinlink]{cleveref}
\usepackage{comment}
\usepackage{subcaption}

\usepackage[normalem]{ulem}

\makeatletter
\newcommand{\removelatexerror}{\let\@latex@error\@gobble}
\makeatother

\makeatletter
\def\@endtheorem{\endtrivlist}
\makeatother

\newtheorem{theorem}{Theorem}
\newtheorem{definition}{Definition}
\newtheorem{proposition}{Proposition}
\newtheorem{remark}{Remark}
\newtheorem{corollary}{Corollary}

\newtheorem{lemma}{Lemma}
\newtheorem{problem}{Problem}
\date{\today}

\newcommand{\nn}{{\mathscr{N}\negthickspace\negthickspace\negthinspace\mathscr{N}}\negthinspace}
\newcommand{\ou}{%
  \mathrel{%
    \vcenter{\offinterlineskip
      \ialign{##\cr$<$\cr\noalign{\kern-1.5pt}$>$\cr}%
    }%
  }%
}%
\newcommand{\xition}[2]{\overset{#2}{_{\scriptscriptstyle\Sigma_#1} \negthickspace\negthickspace\negthinspace\negthinspace\longrightarrow}}

\newcommand{\ltsxition}[2]{\overset{#2}{_{\scriptscriptstyle #1 \vphantom{\scriptstyle (}}\negthickspace\negthickspace\negthinspace\longrightarrow}}

\newcommand{\tllTheta}{\Theta^{\overset{\text{\tiny TLL}}{~}}}

\newcommand{\tllThetaPar}{\Theta^{\overset{\text{\tiny TLL}\parallel}{~}}}

\newcommand{\subcpwa}{_{\negthinspace\text{\tiny CPWA}}}

\renewcommand{\baselinestretch}{0.91}

\begin{document}

% begin nowrap
\title{
Assured Neural Network Architectures for Control and Identification of Nonlinear Systems
} %
\author{James Ferlez\textsuperscript{$*$} and Yasser Shoukry\textsuperscript{$*$}
\thanks{
\textsuperscript{$*$}Department of Electrical Engineering and Computer Science, University of California, Irvine
\texttt{\{jferlez,yshoukry\}@uci.edu}
} %
%\thanks{\textsuperscript{$\ddagger$}Affiliation 2
%\texttt{\{email\}@email.edu}
%}
\thanks{This  work  was  partially  sponsored  by  the  NSF  awards \#CNS-2002405 and \#CNS-2013824.}
} %
%\thanks{
%\textsuperscript{$\dagger$} Equally contributing first authors.
%}
% end nowrap

\maketitle

\begin{abstract}
In this paper, we consider the problem of automatically designing a Rectified Linear Unit (ReLU) Neural Network (NN) architecture (number of layers and number of neurons per layer) with the assurance that it is sufficiently parametrized to control a nonlinear system; i.e. control the system to satisfy a given formal specification. This is unlike current techniques, which provide no assurances on the resultant architecture. Moreover, our approach requires only limited knowledge of the underlying nonlinear system and specification. We assume only that the specification can be satisfied by a Lipschitz-continuous controller with a known bound on its Lipschitz constant; the specific controller need not be known. From this assumption, we bound the number of affine functions needed to construct a Continuous Piecewise Affine (CPWA) function that can approximate any Lipschitz-continuous controller that satisfies the specification. Then we connect this CPWA to a NN architecture using the authors' recent results on the Two-Level Lattice (TLL) NN architecture; the TLL architecture was shown to be parameterized by the number of affine functions present in the CPWA function it realizes.

\end{abstract}

% arXiv submission
% \includecomment{proofs} %
\excludecomment{techreport} %
\excludecomment{arxivreference}

% CDC final submission:
% \excludecomment{proofs} %
% \excludecomment{techreport} %
% \includecomment{arxivreference}

% !TEX root = ./main.tex

\section{Introduction} % (fold)
\label{sec:introduction}

% Recent advances in theory and computation have facilitated wide-spread adoption 
% of Rectified Linear Unit (ReLU) Neural Networks (NNs) for a number of Machine 
% Learning (ML) and Reinforcement Learning (RL) tasks. These advances have in 
% turn spurred an interest in using  NNs in more conventional feedback-control 
% settings, especially for Cyber-Physical Systems (CPSs)~\cite{bojarski2016end}. 
% In particular, an end-to-end learning paradigm is attractive for controlling 
% CPS systems, which are often not amenable to conventional, model-based control 
% methodologies. 

Recent advances in theory and computation have facilitated wide-spread adoption 
of Rectified Linear Unit (ReLU) Neural Networks (NNs) in conventional 
feedback-control settings, especially for Cyber-Physical Systems 
(CPSs)~\cite{bojarski2016end}. However, this proliferation of NN controllers 
has also highlighted weaknesses in state-of-the-art NN design techniques, 
because CPS systems are often \emph{safety critical}. In a safety-critical CPS, 
it is not enough to simply learn a NN controller, i.e. fit data: such a 
controller must also have demonstrable safety or robustness properties, usually 
with respect to \emph{closed-loop} specifications on a dynamical system. 
Moreover, a meaningful safety specification for a CPS is often binary: either a 
controller makes the system safe or it doesn't. This is in contrast to  
optimization-based approaches typical in Machine/Reinforcement Learning  
(ML/RL), where the goal is to optimize a particular cost or reward without any 
requirements imposed on eventual quantity in question. Examples of the former 
include stability about an equilibrium point or forward invariance of a 
particular set of safe states (e.g. for collision avoidance). Examples of the 
latter include minimizing the mean-squared fit error; minimizing regret; or 
maximizing a surrogate for a value function; etc. 

% As a consequence of this distinction, one largely overlooked consideration in 
% modern ML/RL is \emph{NN architecture design} --- 
The distinction between safety specifications and ML/RL objectives is  relevant 
to more than just learning NN weights and biases, though: it is has special 
relevance to \textbf{\itshape NN architecture design} --- i.e. deciding on the 
number of neurons and their connection (or arrangement) in the NN to be trained 
in the first place. Specifically, (binary) safety specifications beg 
\emph{existential questions} about NN architectures in a way that conventional, 
optimization-focused ML/RL techniques do not: given a particular NN 
architecture for a controller, it is necessary to ask whether there is there 
\emph{any} possible choice of weights and biases that achieve the desired 
safety specification. By contrast, conventional ML/RL type problems instead 
take an architecture as given, and attempt to achieve the best training error, 
reward, etc. subject to that implicit constraint. Thus, in typical ML/RL 
treatments, NN architectures merely adjust the final cost/reward, 
rather than leading to ill-posedness, such as can occur with a safety 
specification.

In this paper, we directly address the issue of whether a ReLU NN architecture 
is well-posed as a state-feedback controller for a given nonlinear system and 
closed-loop (safety) specification. That is we present a systematic methodology 
for designing a NN controller architecture that is \emph{guaranteed} (or  
\emph{assured}) to be able to meet a given closed-loop specification. Our 
approach provides such a guarantee contingent on the following properties of 
the system and specification: 
\begin{enumerate}[{\itshape (i)}]
	\item the nonlinear system's vector field is Lipschitz continuous with 
		\emph{known} Lipschitz constants;

	\item there exists a Lipschitz continuous controller that satisfies the 
		closed-loop specification \emph{robustly}, and the Lipschitz constant 
		of that controller is known (\emph{\underline{although the controller 
		itself need not be known}}); and

	\item the conjectured Lipschitz continuous controller makes a compact 
		subset of the state space positive invariant.
\end{enumerate}
The need to assume the existence of a controller is primarily to ensure that 
the specification is well-posed in general -- i.e., for \emph{any} controller, 
whether it is a NN or not; we will elaborate on the robustness in \emph{(ii)} 
subsequently. Importantly, subject to these conditions, our approach can design 
a NN controller architecture with the following assurance: there exists neuron 
weights/biases for that architecture such that it can exactly meet the same 
specification as the assured non-NN controller (albeit non-robustly). Moreover, 
our proposed methodology requires \emph{only} the information described above, 
so it is
% knowledge of a bound on the conjectured controller's Lipschitz 
% constant; the robustness of the specification; the Lipschitz constants of the 
% nonlinear system; and a bound on the vector field over the positively invariant 
% set of states.
applicable even without perfect knowledge of the underlying system dynamics, 
albeit at the expense of designing rather larger architectures.

%In this paper, we focus on the fundamental question of how to systematically choose the NN architecture (number of layers and number of neurons per layer) such that we guarantee the correctness of the chosen NN architecture in terms of its ability to control a nonlinear dynamical system. In particular, we seek to use knowledge of the underlying control problem to guide the design of NN architectures.

%Different from the current state-of-the-art 

%One of the outstanding problems with data-driven approaches is that the architecture for the NN is chosen either according to heuristics or else via a computationally expensive iteration scheme that involves adapting the architecture iteratively and re-training the NN. Besides being computationally taxing, neither of these provide any assurances that the resultant architecture is sufficient to adequately control the underlying system, either in terms of performance or stability. 

The cornerstone of our approach is a special ReLU NN architecture introduced by 
the authors in the context of another control problem: viz. the Two-Level 
Lattice (TLL) ReLU NN architecture \cite{FerlezAReNAssuredReLU2020}. A TLL NN, 
like all ReLU NNs, instantiates a Continuous, Piecewise Affine (CPWA) function, 
and hence implements one of finitely many \emph{local linear 
functions}\footnote{The term ``linear'' here is somewhat of a misnomer: 
``affine'' is more accurate. However, we use this terminology for consistency 
with the literature.} at each point in its domain; thus, its domain, like that 
of any CPWA, can be partitioned into \emph{linear} regions, each of which 
corresponds to a different local linear function. However, unlike general ReLU 
NN architectures, a TLL NN exposes these local linear functions directly in the 
parameters of the network. As a consequence of this parameterization, then, TLL 
NNs are easily parameterized by their number of linear regions: after all, each 
linear region must instantiate at least one local linear function -- see 
\cite[Theorem 2]{FerlezAReNAssuredReLU2020}. In fact, this idea also applies to 
\emph{upper bounds} on the number of regions desired in an architecture; see 
\cite[Theorem 3]{FerlezAReNAssuredReLU2020}. TLL NNs thus have the special 
property that they can be used to connect a \textbf{\itshape desired number of 
linear regions directly to a ReLU architecture.} 

Thus, to obtain an assured controller architecture, it is enough to obtain an 
assured upper-bound on the number of linear regions required of a controller. 
%to achieve the desired control specification.
In this paper, we show that such an assured upper bound can be obtained using 
just the information assumed in \emph{(i) - (iii)}; i.e. primarily Lipschitz 
constants and bounds on the relevant objects. This bound is derived by counting 
the number of linear regions needed to linearly and continuously interpolate 
between regularly spaced ``samples'' of a %Lipschitz-continuous 
controller with known Lipschitz constant.
% We reiterate, though, that 
% \emph{specific} details of neither vector field nor the assumed-to-exist 
% controller are required to produce such a bound: only the bounds and Lipschitz 
% constants are necessary. This design methodology is thus applicable to the 
% end-to-end learning paradigms, which are often deployed exactly when detailed 
% model information is scarce or impossible to obtain.

% Using this Lipschitz-constant bound -- \emph{but no other specific information 
% about the controller} -- together with the Lipschitz constants/vector field 
% bound of the system and robustness of the specification, we exhibit an 
% upper-bound for the number of activation regions needed to approximate this 
% controller by a CPWA controller, while still meeting the same specifications in 
% closed loop.

Moreover, this core approach is relevant to end-to-end learning beyond just the 
design of \emph{controller} architectures: it can also be used to obtain 
architectures that are guaranteed to represent the dynamics of a nonlinear 
system itself -- i.e. assured architectures for \emph{system identification}. 
Indeed, in this paper, we further show that contingent on information akin to 
\emph{(i)} and \emph{(iii)}, it is possible to generate a ReLU architecture 
that is guaranteed to capture the essence of an \emph{unknown} nonlinear 
control system. Specifically, we exhibit a methodology for designing an 
architecture to represent a nonlinear (controlled) vector field with the 
following assurance: if the ReLU vector field is sufficiently well trained on 
data from a compatible -- but unknown -- controlled vector field, then robustly 
controlling the ReLU dynamics to a specification will yield a controller that  
likewise controls the unknown dynamics to the same specification (albeit 
non-robustly). Providing a guaranteed architecture for system identification 
has unique value for end-to-end learning: because the ReLU control system can 
be used as a surrogate for the original nonlinear system, control design can be 
moved entirely from the \emph{unknown} system to the \emph{known} ReLU 
surrogate, the latter of which can be \emph{simulated} instead. Furthermore, 
this system identification can be combined with a guaranteed controller 
architecture to do \emph{in silico} ReLU control design.

The contributions of this paper can be summarized thusly.
\begin{enumerate}[{\itshape I.}]
	\item A new notion of \emph{Abstract Disturbance Simulation (ADS)} to 
		formulate of robust specification satisfaction; ADS unifies and 
		generalizes several related notions of simulation in the literature -- 
		see \cref{sub:perturbation_}.

	\item A methodology to design a ReLU architecture that is  
		\emph{assured} to be able to control an unknown nonlinear system to 
		meet a closed-loop specification; this is subject to the existence of a 
		(likewise) \emph{unknown} controller that \emph{robustly} meets the 
		same specification.

	\item A methodology to design a ReLU architecture that can be used in 
		system identification of an unknown nonlinear control system; this 
		architecture, when adequately trained, is \emph{assured} to be viable 
		as a surrogate for the original nonlinear system in controller design.
\end{enumerate}
A preliminary version of this paper appeared  
as~\cite{FerlezTwoLevelLatticeNeural2020a}. Relative  
to~\cite{FerlezTwoLevelLatticeNeural2020a}, this paper has the following 
additional novel content: first, it uses new, dramatically improved techniques 
to obtain a smaller architecture than~\cite{FerlezTwoLevelLatticeNeural2020a} 
(see \cref{rem:improved_bound}); second, it includes full proofs of every 
claim; and third, it contains the extension to system identification 
architectures noted above.

\subsection{Related Work} % (fold)
\label{sub:related_work}

% In this section, we organize the rather scattered literature on NN architecture 
% design by assessing methods along three dimensions: the \emph{automaticity} of 
% the method; the \emph{assurances} provided; and finally its 
% \emph{generalizability}. Automaticity refers to how much manual effort is 
% required to obtain a functioning architecture; this dimension spans from fully 
% automatic (with perhaps some meta-hyperparameters) to fully  manual, bespoke 
% architecture designs. Assurances refer to the flexibility in the type of 
% assurances for which a method can still successfully design assured 
% architectures; this dimension spans from no assurances possible at all to 
% assurances of only stability to somewhat more complicated safety 
% specifications. Generalizability refers to the extent to which a method depends 
% on specific assumptions and knowledge about the dynamics, specifications, etc.; 
% this dimension spans from very general, purely data-driven methods to those 
% that work only for one specific, parameterized  type of dynamics. 

The literature most directly relevant to this paper is work by the authors: 
AReN \cite{FerlezAReNAssuredReLU2020} and the preliminary version of this 
paper, \cite{FerlezTwoLevelLatticeNeural2020a}. The former contains an  
algorithm that generates an architecture assured to represent an optimal MPC 
controller. The AReN algorithm is fully automatic, and the assurances provided 
are the same as those for the referent MPC controller; but it is not very 
generalizable, given the restriction to MPC control. This paper and its 
preliminary version offer a significantly improved methodology. \emph{The  
architecture design presented herein is likewise fully automatic; however, it 
is generalizable to any Lipschitz continuous control system, and it provides 
assurances for a wide variety specification captured by simulation relations.}

The largest single class of NN architecture design algorithms
% -- and the one most 
%literally deserving of that nomenclature -- 
is commonly referred to as Neural Architecture Search (NAS) algorithms. These 
algorithms essentially use an iterative improvement/optimization scheme to 
design a NN architecture, thus automating something like a 
`guess-train-evaluate'' hyperparameter tuning loop. There is a large literature 
on NAS algorithms; \cite{ElskenNeuralArchitectureSearch2019} is a good summary.
%and such algorithms are 
% distinguished by core three aspects \cite{ElskenNeuralArchitectureSearch2019}: 
% a \emph{Search Space} is updated according to a \emph{Search Strategy}, which 
% employs a mechanism of \emph{Performance Evaluation} -- generally some kind of 
% training -- to prefer some architectures over others (from the Search Space). 
Typical NASs design architectures within some structured class of  
architectures such as: ``chain'' architectures (i.e. a sequence of fully 
connected layers) \cite{MendozaAutomaticallyTunedNeuralNetworks2016}; chain 
architectures with different layer types (e.g. convolution, pooling, etc. in 
addition to fully connected) \cite{BakerAcceleratingNeuralArchitecture2017, 
CaiEfficientArchitectureSearch2017}; or mini-NN architectures that are 
replicated and interconnected to form larger networks 
\cite{ZophLearningTransferableArchitectures2018}. They then update these 
architectures according to a variety of different mechanisms: RL formulations
%, where the action is choosing an architecture, and the 
% reward is obtained by training that architecture 
\cite{baker2016designing, ZophNeuralArchitectureSearch2017, 
ZophLearningTransferableArchitectures2018}; sequential decision problems 
formulations where actions corresponding to network morphisms 
\cite{CaiEfficientArchitectureSearch2017}; Bayesian optimization formulations
%, where a parameterized density is used as a 
% prior to generate samples of architectures, which are in turn used as data to 
% update the prior parameters 
\cite{bergstra2012random, MendozaAutomaticallyTunedNeuralNetworks2016}; and 
Neuro-evolutionary approaches (relatedly population dynamics or genetic 
algorithms) 
%, where individual elements in a ``population'' of architectures 
% replicate according to their fitness 
\cite{ElskenEfficientMultiObjectiveNeural2018, 
RealRegularizedEvolutionImage2019a}. Different evaluation mechanisms are used  
to evaluate the ``quality'' of current architecture iterate: lower fidelity 
models 
% -- e.g. shorter training times 
\cite{ZophLearningTransferableArchitectures2018} or weight inheritance 
morphisms \cite{ElskenNeuralArchitectureSearch2019, 
CaiEfficientArchitectureSearch2017} (a sort of transfer learning); learning 
curve extrapolation to estimate the final performance of an architecture before 
training has converged \cite{BakerAcceleratingNeuralArchitecture2017}; and 
one-shot models that agglomerate many architectures into a single large 
architecture that shares edges between individual architectures 
\cite{XieSNASStochasticNeural2018}. Notably these algorithms all share the  
same features: %in terms of our assessment scheme:
they are highly automatic, even accounting for the need to choose 
meta-hyperparameters; they are fairly general, since they are data-driven; %\footnote{However, this outsources 
%the problem of  collecting enough representative data to get good 
%architectures.}; 
however, they provide no closed-loop assurances. %, since they are purely 

At the opposite end of our assessment spectra are control-based methods for 
obtaining NN controllers with assurances; we regard these methods as implicit 
architecture design techniques, since exhibiting a NN controller serves as a 
direct validation of that controller's architecture. Examples of these methods 
include: directly approximating a controller by a NN for non-affine systems 
\cite{ZhangDirectNNapproximationControl2020}; adaptively learning NN controller 
weights to ensure Input-to-State stability for certain third-order affine 
systems  \cite{WangLearningISSModularAdaptive2012,GeDirectAdaptiveNN2002}; NN 
hybrid adaptive control for stabilizing uncertain impulsive dynamical systems 
\cite{HayakawaNeuralNetworkHybrid2008}. These methods are based on the 
assertion that a function of interest can be approximated by a sufficiently 
large (usually shallow) NN: the size of this NN is explicit, which 
limits their effectiveness as architecture design methods. Even neglecting this 
shortcoming, these methods generally provide just one meaningful assurance 
(stability); they are not at all general, since they are based on approximating 
a specific, hand-designed controller; and they are thus highly manual methods.

A subset of NN \emph{verification} methods from the control system literature 
is related to the ``guess-train-evaluate''  architecture design iterations 
described above. In particular, some closed-loop NN verifiers provide 
additional dynamical information about how a NN controller fails to meet a 
specification. Examples include: using complementary analysis on NN-controlled 
linear systems, thus obtaining sufficient conditions for stability in terms of 
LMIs \cite{AydinogluStabilityAnalysisComplementarity2021}; training and 
verifying a specific NN architecture as a barrier certificate for hybrid 
systems \cite{ZhaoSynthesizingReLUNeural2021}; and using adversarial 
perturbation to verify NN control policies \cite{WangNeuralNetworkControl2020}. 
These methods  can be assessed as follows: they are highly automatic 
(verifiers); 
%\footnote{Although \cite{ZhaoSynthesizingReLUNeural2021} requires 
%piecewise-linear approximations of the dynamics, which may require manual 
%effort.};
they are of limited generalizability, since they require either specific models 
and/or detailed knowledge of the dynamical model; and each provides one and 
only one assurance (e.g. stability or a barrier certificate).  A similar, but 
less applicable, subset of the control literature that consists of experimental 
work that suggests promising NN controller architectures. These works do not do 
automatic NN architecture design, nor do they contain verification algorithms  
of the type suggested above. Even so, they experimentally support using some 
conventional NN architectures as controllers 
\cite{PonKumarDeepLearningArchitecture2018}; and using Input Convex NNs  
(ICNNs) for controllers and system identification 
\cite{ChenOptimalControlNeural2018}.% (ICNNs are also related to TLL NNs, the 
%latter of which are built out of convex NN components -- i.e.  $\min$ terms).

We also acknowledge prior work on system identification using NNs, although 
they do not emphasize architecture design. These methods suffer from a lack of 
assurances on the resultant NNs/architectures 
\cite{CHENNonlinearSystemIdentification1990, 
OgunmoluNonlinearSystemsIdentification2016,  
FaltSystemIdentificationHybrid2019, 
Velazquez-VelazquezHybridRecurrentNeural2011}. 

Finally, subsequent to the publication of 
\cite{FerlezTwoLevelLatticeNeural2020a}, we became aware of other works that 
use more or less what we describe as TLL NNs  
\cite{SunGeneralLatticeRepresentation2011, HeReluDeepNeural2020}. The former is 
concerned with simplification of explicit MPC controllers rather than NN 
architecture design; the latter can be regarded as architecture design for a 
different application (although the architectures are less efficient than the  
ones presented here -- see \cref{rem:fem_remark}).

%

% !TEX root = ./main.tex

\section{Preliminaries} % (fold)
\label{sec:prelims}

\subsection{Notation} % (fold)
\label{sub:notation}

We denote by $\mathbb{N}$, $\mathbb{R}$ and $\mathbb{R}^+$ the set of natural 
numbers, the set of real numbers and the set of non-negative real numbers, 
respectively. For a function $f : A \rightarrow B$, let $\text{dom}(f)$ return 
the domain of $f$, and let $\text{range}(f)$ return the range of $f$. For $x 
\in \mathbb{R}^n$, we will denote by $\lVert x \rVert$ the max-norm of $x$. 
Relatedly, for $x \in \mathbb{R}^n$ and $\epsilon \geq 0$ we will denote by 
$B(x;\epsilon)$ the open ball of radius $\epsilon$ centered at $x$ as specified 
by $\lVert \cdot \rVert$, and $\overline{B}(x;\epsilon)$ its  closed-ball 
analog. Let $\text{int}(X)$ denote the interior of a set  $X\subseteq 
\mathbb{R}^n$, and $\text{bd}(X) = X \backslash \text{int}(X)$ denote its 
boundary. $e_i$ will denote the $i^\text{th}$ column of the $(n \times n)$ 
identity matrix, unless otherwise specified. Let $\text{hull}(X)$ denote the 
convex hull of a set of points $X\subseteq \mathbb{R}^n$. For $f : \mathbb{R}^n 
\rightarrow \mathbb{R}^m$, $\lVert f \rVert_\infty \triangleq  
\text{ess}\sup_{x\in \mathbb{R}^n} \lVert f(x) \rVert$, and $\lVert f 
\rVert_{X}$ will denote the same but restricted to the set $X$. The projection 
map over $\mathbb{R}^n$ will be denoted by  $\pi_i : \mathbb{R}^n \rightarrow 
\mathbb{R}$, so that $\pi_i(x)$ returns the  $i^\text{th}$ component of the 
vector $x$ (in the understood coordinate  system). Finally, given two sets $A$ 
and $B$ denote by $B^A$ the set of all functions $f: A \rightarrow B$.
% with domain $A$ and range $B$.

% subsection notation (end)

\subsection{Dynamical Model} % (fold)
\label{sub:dynamical_model}
In this paper, we will assume an underlying, but not necessarily known,  
continuous-time nonlinear dynamical system specified by an ordinary 
differential equation (ODE): that is
\begin{equation}\label{eq:main_ode}
	\dot{x}(t) = f( x(t) , u(t))
\end{equation}
where the state vector $x(t) \in \mathbb{R}^n$ and the control vector $u(t) \in 
\mathbb{R}^m$. Formally, we have the following definition:
\begin{definition}[Control System]
\label{def:control_system}
	A \textbf{control system} is a tuple $\Sigma = (X, U, \mathcal{U}, f)$ 
	where

	\begin{itemize}
		\item $X \subset \mathbb{R}^n$ is the connected, compact subset of 
			the state space with non-empty interior;

		\item $U \subset \mathbb{R}^m$ is the compact set of admissible 
			controls;

		\item $\mathcal{U} \subseteq U^{\mathbb{R}^+}$ is the space of 
			admissible open-loop control functions -- i.e. $v \in \mathcal{U}$ 
			is a function $v : \mathbb{R}^+ \rightarrow U$; and

		\item $f : \mathbb{R}^n \times U \rightarrow \mathbb{R}^n$ is a 
			vector field specifying the time evolution of states according to 
			\eqref{eq:main_ode}.
	\end{itemize}
	A control system is said to be (globally) \textbf{Lipschitz} if there 
	exists constants $K_x$ and $K_u$ s.t. for all $x,x^\prime \in\mathbb{R}^n$ 
	and $u,u^\prime \in \mathbb{R}^m$: 
	\begin{equation}\label{eq:dynamics_lipshitz}
		\lVert f(x,u) - f(x^\prime, u^\prime) \rVert \leq 
			K_x \lVert x - x^\prime \rVert + K_u \lVert u - u^\prime \rVert.
	\end{equation}
	% For a Lipschitz control system, the following vector field bound is well 
	% defined:
	% \begin{equation}
	% 	\mathcal{K} \triangleq \max_{x\in X, u \in U} \lVert f(x,u) \rVert.
	% \end{equation}
\end{definition}

In the sequel, we will primarily be concerned with solutions to 
\eqref{eq:main_ode} that result from instantaneous state-feedback controllers, 
$\Psi : X \rightarrow U$. Thus, we use $\zeta_{x_0 \Psi}$ to denote the 
\emph{closed-loop} solution of \eqref{eq:main_ode} starting from initial 
condition $x_0$ (at time $t=0$) and using \emph{state-feedback controller} 
$\Psi$. We refer to such a $\zeta_{x_0\Psi}$ as a (closed-loop) 
\emph{trajectory} of its associated control system.

\begin{definition}[Closed-loop Trajectory]
	Let $\Sigma$ be a Lipschitz control system, and let $\Psi : \mathbb{R}^n 
	\rightarrow U$ be a globally Lipschitz continuous function. A 
	\textbf{closed-loop trajectory} of $\Sigma$ under controller $\Psi$ and  
	starting from $x_0 \in X$ is the function $\zeta_{x_0\Psi} : \mathbb{R}^+ 
	\rightarrow X$ that solves the integral equation:
	\begin{equation}\label{eq:feedback_integral_eq}
		\zeta_{x_0\Psi}(t) = x_0 + \int_0^t f( \zeta_{x_0\Psi}(\sigma), \Psi( \zeta_{x_0\Psi}(\sigma) )) d\sigma.
	\end{equation}
	It is well known that such solutions exist and are unique under these 
	assumptions \cite{KhalilNonlinearSystems2001}.
	% We will only consider 
	% feedback controllers for which $X$ is positively invariant under feedback, 
	% i.e. $\text{range}(\zeta_{x_0\Psi}) \subseteq X$.
\end{definition}
\begin{definition}[Feedback Controllable]
	A Lipschitz control system $\Sigma$ is \textbf{feedback controllable} by a 
	Lipschitz controller $\Psi: \mathbb{R}^n \rightarrow U$ if the following is 
	satisfied:
	\begin{equation}
		\Psi \circ \zeta_{x\Psi} \in \mathcal{U} \qquad \forall x\in X.
	\end{equation} 
	If $\Sigma$ is feedback controllable for any such $\Psi$, then we simply 
	say that it is feedback controllable.
	% A Lipschitz control system 
	% is called feedback controllable if it is feedback controllable for each 
	% globally Lipschitz feedback controller.
\end{definition}
Because we're interested in a compact set of states, $X$, we consider only 
feedback controllers whose closed-loop trajectories stay within $X$.
\begin{definition}[Positive Invariance]
	A feedback trajectory of a Lipschitz control system, $\zeta_{x_0\Psi}$, is 
	\textbf{positively invariant} if $\zeta_{x_0\Psi}(t) \in X$ for all $t \geq 
	0$. A controller $\Psi$ is positively invariant if $\zeta_{x_0\Psi}$ is 
	positively invariant for all $x_0 \in X$.
\end{definition}
For technical reasons, we will also need the following stronger notion of 
positive invariance.
\begin{definition}[$\delta$,$\tau$ Positive Invariance]
\label{def:delta_tau_positive_invariance}
	Let $\delta,\tau \negthinspace > \negthinspace 0$ and 
	$\text{edge}_\delta(X) \negthinspace \triangleq \negthinspace \cup_{x \in X 
	\backslash \text{int}(X)} ( X \negthinspace \cap \negthinspace B(x;\delta) 
	)$. Then a positively invariant controller $\negthinspace \Psi 
	\negthinspace$ is $\mathbf{\delta}$,$\mathbf{\tau}$ \textbf{positively 
	invariant} if \vspace{-0.2em}
	\begin{equation}
		\forall x_0 \in \text{edge}_{\delta}(X)
		\; . \; 
		\zeta_{x_0\negthinspace\Psi}(\tau) \negthinspace \in \negthinspace X \backslash \text{edge}_{\delta}(X)
	\end{equation}
	\vspace{-14pt}

	\noindent and $\Psi$ is positively invariant with respect to $X \backslash 
	\text{edge}_\delta(X)$.
\end{definition}
For a $\delta$,$\tau$ positively invariant controller, trajectories that start 
$\delta$-close to the boundary of $X$ will end up $\delta$-far 
\emph{away} from that boundary after $\tau$ seconds, and remain there forever 
after.

% For any given feedback controller, $\Psi$, the open-loop control functions 
% created by its trajectories may not be elements of $\mathcal{U}$. Thus, we make 
% the following additional definition:

% \begin{remark}
% 	In this paper, we will henceforth consider only feedback controllable 
% 	Lipschitz control systems. 
% \end{remark}

Finally, borrowing from \cite{ZamaniSymbolicModelsNonlinear2012}, we define a 
$\tau$-sampled transition system embedding of a feedback-controlled system.
\begin{definition}[$\tau$-sampled Transition System Embedding]
\label{def:embedding}
	Let $\Sigma=(X,U,\mathcal{U},f)$ be a feedback controllable Lipschitz 
	control system, and let $\Psi : \mathbb{R}^n \rightarrow U$ be a Lipschitz 
	continuous feedback controller. For any $\tau > 0$, the 
	$\mathbf{\tau}$\textbf{-sampled transition system embedding} of $\Sigma$ 
	under $\Psi$ is the tuple $S_\tau(\Sigma_\Psi) = (X_\tau, \mathcal{U}_\tau, 
	\xition{\Psi}{~} )$ where:

	\begin{itemize}
		\item $X_\tau = X$ is the state space;

		\item $\mathcal{U}_\tau = \{ (\Psi \circ \zeta_{x_0\Psi})|_{t\in 
			[0,\tau]}: x_0 \in X \}$ is the set of open loop control inputs 
			generated by $\Psi$-feedback, each restricted to the domain 
			$[0,\tau]$; and

		\item $\xition{\Psi}{~} \subseteq X_\tau \times \mathcal{U}_\tau 
			\times X_\tau$ such that $x \xition{\Psi}{u} x^\prime$ iff \\ 
			$\qquad$ both $u = (\Psi \circ \zeta_{x\Psi})|_{t\in [0,\tau]}$ and 
			$x^\prime = \zeta_{x\Psi}(\tau)$. 
	\end{itemize}
	$S_\tau(\Sigma_\Psi)$ is thus a \textbf{metric} transition system 
	\cite{ZamaniSymbolicModelsNonlinear2012}.
	% \footnote{For our purposes, a 
	% metric transition system is one whose state space is contained in a metric 
	% space.}.
\end{definition}

\begin{definition}[Simulation Relation]
\label{def:simulation_relation}
Let $S_1 = (X_1, \mathcal{U}_1, \ltsxition{1}{~} )$ and $S_2 =  (X_2, 
\mathcal{U}_2,\ltsxition{2}{~} )$ be two metric transition systems. Then we say 
that $S_2$ simulates $S_1$, written $S_1 \precsim S_2$, if there exists a 
relation $\precsim \subseteq X_1 \times X_2$ such that

\begin{itemize}
	\item $\{x_1 | (x_1,x_2) \in \precsim\} = X_1$; and 

	\item for all $(x_1,x_2) \in \precsim$ we have
		\begin{align}
			&x_1 \ltsxition{1}{u} 
					x_1^\prime \implies \notag \\
			&\qquad\qquad \exists x_2^\prime \in X_2 . \big(  
				(x_1^\prime, x_2^\prime) \in \precsim \; \wedge \; x_2  
				\ltsxition{2}{u} x_2^\prime
				\big).
				\label{eq:ordinary_simulation}
		\end{align}
\end{itemize}
\end{definition}

\subsection{ReLU Neural Network Architectures} % (fold)
\label{sub:relu_neural_network_architectures}
We will consider controlling the nonlinear system defined in 
\eqref{eq:main_ode} with a state-feedback neural network controller $\nn$:
\begin{equation}
	\nn: X \rightarrow U
\end{equation}
where $\nn$ denotes a Rectified Linear Unit Neural Network (ReLU NN).
Such a ($K$-layer) ReLU NN is specified by composing $K$ \emph{layer} functions 
(or just \emph{layers}). A layer with $\mathfrak{i}$ inputs and $\mathfrak{o}$ 
outputs is specified by a $(\mathfrak{o} \times \mathfrak{i} )$ matrix of 
\emph{weights}, $W$, and a $(\mathfrak{o} \times 1)$ matrix of \emph{biases}, 
$b$, as follows:
\begin{align}
	L_{\theta} : \mathbb{R}^{\mathfrak{i}} &\rightarrow \mathbb{R}^{\mathfrak{o}} \notag\\
	      z &\mapsto \max\{ W z + b, 0 \}
\end{align}
where the $\max$ function is taken element-wise, and $\theta \triangleq (W,b)$ 
for brevity. Thus, a $K$-layer ReLU NN function %as above
is specified by $K$ layer functions $\{L_{\theta^{(i)}} : i = 1, \dots, K\}$ 
whose input and output dimensions are \emph{composable}: that is, they satisfy 
$\mathfrak{i}_{i} = \mathfrak{o}_{i-1}$ for $i = 2, \dots, K$. Specifically:
\begin{equation}
	\nn(x) = (L_{\theta^{(K)}} \circ L_{\theta^{(K-1)}} \circ \dots \circ L_{\theta^{(1)}})(x).
\end{equation}
When we wish to make the dependence on parameters explicit, we will index a 
ReLU function $\nn$ by a \emph{list of matrices} $\Theta \triangleq ( 
\theta^{(1)}, \dots , \theta^{(K)} )$ \footnote{That is, $\Theta$ is not the 
concatenation of the $\theta^{(i)}$ into a single large matrix, so it preserves 
information about the sizes of the constituent $\theta^{(i)}$.}.

Fixing the number of layers and the \emph{dimensions} of the associated 
matrices $\theta^{(i)} = (\; W^{(i)}, b^{(i)}\; )$ specifies the 
\emph{architecture} of a fully-connected ReLU NN. Therefore, we will use:
\begin{equation}
	\text{Arch}(\Theta) \negthinspace \triangleq \negthinspace ( (n,\mathfrak{o}_{1}), (\mathfrak{i}_{2},\mathfrak{o}_{2}), \ldots, 
	%(\mathfrak{i}_{K\negthinspace-\negthinspace 1},\mathfrak{o}_{K\negthinspace-\negthinspace 1}),
	(\mathfrak{i}_{K}, m))
\end{equation}
to denote the architecture of the ReLU NN $\nn_{\Theta}$.

Since we are interested in designing ReLU architectures, we will also need the 
following result from \cite[Theorem 7]{FerlezAReNAssuredReLU2020}, which states 
that a Continuous, Piecewise Affine (CPWA) function, $\mathsf{f}$, can be 
implemented exactly using a Two-Level-Lattice (TLL) NN architecture that is 
parameterized by the local linear functions in $\mathfrak{f}$.

\begin{definition}[Local Linear Function]
	Let $\mathsf{f} : \mathbb{R}^n \rightarrow \mathbb{R}^m$ be CPWA. Then a 
	\textbf{local linear function of} $\mathsf{f}$ is a linear function $\ell : 
	\mathbb{R}^n \rightarrow \mathbb{R}^m$ if there exists an open set 
	$\mathfrak{O}\subseteq \mathbb{R}^n$ such that $\ell(x) = \mathsf{f}(x)$ 
	for all $x\in \mathfrak{O}$.
\end{definition}

\begin{definition}[Linear Region]\label{def:linear_region}
	Let $\mathsf{f} : \mathbb{R}^n \rightarrow \mathbb{R}^m$ be CPWA. Then a 
	\textbf{linear region of} $\mathsf{f}$ is the largest set $\mathfrak{R} 
	\subseteq \mathbb{R}^n$ such that $\mathsf{f}$ has only one local linear 
	function on $\text{int}(\mathfrak{R})$.
\end{definition}

\begin{theorem}[Two-Level-Lattice (TLL) NN Architecture {[}7, Theorem 7{]}]
\label{thm:tll_architecture}
	Let $\mathsf{f}:\mathbb{R}^n \rightarrow \mathbb{R}^m$ be a CPWA function, 
	and let $\bar{N}$ be an upper bound on the number of local linear functions 
	in $\mathsf{f}$. Then there is a Two-Level-Lattice (TLL) NN architecture 
	$\text{Arch}(\tllTheta_{\bar{N}})$ parameterized by $\bar{N}$ and values of 
	$\tllTheta_{\bar{N}}$ such that:
	\begin{equation}
		\mathsf{f} = \nn_{\negthickspace\tllTheta_{\bar{N}}}.
	\end{equation}
	In particular, the number of linear regions of $\mathsf{f}$ is such an 
	upper bound on the number of local linear functions.
\end{theorem}
In this paper, we will find it convenient to use \cref{thm:tll_architecture} to 
create TLL architectures \emph{component-wise}. To this end, we define the 
following notion of NN composition.
\begin{definition}
	\label{def:parallel_composition}
	Let $\nn_{\Theta_{\scriptscriptstyle 1}}$ and 
	$\nn_{\Theta_{\scriptscriptstyle 2}}$ be two $K$-layer NNs with parameter 
	lists:
	\begin{equation}
		\Theta_i = ((W^{\scriptscriptstyle |1}_i, b^{\scriptscriptstyle |1}_i), \dots, (W^{\scriptscriptstyle |K}_i, b^{\scriptscriptstyle |K}_i)), \quad i = 1,2.
	\end{equation}
	Then the \textbf{parallel composition} of $\nn_{\Theta_{\scriptscriptstyle 
	1}}$ and $\nn_{\Theta_{\scriptscriptstyle 2}}$ is a NN given by the 
	parameter list
	\begin{equation}
		\Theta_{1} \parallel \Theta_{2} \triangleq \big(\negthinspace
			\left(
				\negthinspace
				\left[
					\begin{smallmatrix}
						W^{\scriptscriptstyle |1}_1 \\
						W^{\scriptscriptstyle |1}_2
					\end{smallmatrix}
				\right],
				\left[
					\begin{smallmatrix}
						b^{\scriptscriptstyle |1}_1 \\
						b^{\scriptscriptstyle |1}_2
					\end{smallmatrix}
				\right]
				\negthinspace
			\right),
			{\scriptstyle \dots},
			\left(
				\negthinspace
				\left[
					\begin{smallmatrix}
						W^{\scriptscriptstyle |K}_1 \\
						W^{\scriptscriptstyle |K}_2
					\end{smallmatrix}
				\right],
				\left[
					\begin{smallmatrix}
						b^{\scriptscriptstyle |K}_1 \\
						b^{\scriptscriptstyle |K}_2
					\end{smallmatrix}
				\right]
				\negthinspace
			\right)
		\negthinspace\big).
	\end{equation}
	That is $\Theta_{1} \negthickspace \parallel \negthickspace \Theta_{2}$ 
	accepts an input of the same size as (both) $\Theta_1$ and $\Theta_2$, but 
	has as many outputs as $\Theta_1$ and $\Theta_2$ combined.
\end{definition}
\begin{corollary}
\label{cor:parallel_tll}
	Let $\mathsf{f}: \mathbb{R}^n \rightarrow \mathbb{R}^m$ be a CPWA function, 
	each of whose $m$ \textbf{component} CPWA functions are denoted by 
	$\mathsf{f}_i : x \mapsto \pi_i(\mathsf{f}(x))$, and let $\bar{N}$ be an 
	upper bound on the number of linear regions in each $\mathsf{f}_i$.

	Then for $\tllTheta_{\bar{N}}$ a TLL architecture representing a CPWA  
	$\mathbb{R}^n\rightarrow \mathbb{R}$, the $m$-fold \textbf{parallel} TLL 
	architecture:
	\begin{equation}
		\tllThetaPar_{^n_m\negthinspace\bar{N}}
		\triangleq
		\tllTheta_{\bar{N}}
			\parallel
		\dots
			\parallel
		\tllTheta_{\bar{N}}
	\end{equation}
	has the property that $\mathsf{f} = \nn_{\tllThetaPar_{^n_m\negthinspace\bar{N}}}$.
\end{corollary}
\begin{proof}
	Apply \cref{thm:tll_architecture} component-wise.
\end{proof}
% We refer the reader to \cite{FerlezAReNAssuredReLU2020} for more details.

Finally, note that a ReLU NN function, $\nn$, is known to be a continuous, 
piecewise affine (CPWA) function consisting of finitely many linear segments. 
Thus, a function $\nn$ is itself necessarily globally Lipschitz continuous.
% subsection relu_neural_network_architectures (end)

\subsection{Notation Pertaining to Hypercubes} % (fold)
\label{sub:notation_pertaining_to_hypercubes}
Since the unit ball of the max-norm, $\lVert \cdot \rVert$, on $\mathbb{R}^n$ 
is a hypercube, we will make use of the following notation.
\begin{definition}[Face/Corner of a hypercube]
\label{def:face_corner}
	Let $C_n = [0,1]^n$ be a unit hypercube of dimension $n$. A set $F 
	\subseteq C_n$ is a $k$-dimensional \textbf{face} of $C_n$ if there exists 
	a set $J \subseteq \{1, \negthinspace \dots, \negthinspace n\}$ such that 
	$|J| = n \negthinspace - \negthinspace k$ and
	\begin{equation}
		\forall x \in F ~.~ \bigwedge_{j \in J} \Big( \pi_j(x) \in \{0,1\} \Big).
	\end{equation}
	Let $\mathscr{F}_k(C_n)$ denote the set of $k$-dimensional faces of $C_n$, 
	and let $\mathscr{F}(C_n)$ denote the set of all faces of $C_n$ (of any 
	dimension). A \textbf{corner} of $C_n$ is a $0$-dimensional face of $C_n$. 
	Furthermore, we will use the notation $F_i^k$ to denote an  
	\textbf{full-dimensional face} ($n-1$-dimensional in this case) whose index 
	set $J = \{i\}$ and whose projection on the $i^\text{th}$ coordinate is $k 
	\in \{0,1\}$.
	% \begin{equation}
	% 	F_i^k \negthinspace \triangleq \negthinspace \{ x \in C | \pi_i(x) = k\} \text{ for } i \in \{1 \dots n\}, k \in \{0,1\}.
	% \end{equation}

	We extend these definitions by isomorphism to any other hypercube in 
	$\mathbb{R}^n$.
\end{definition}
% subsection notation_pertaining_to_hypercubes (end)

% section prelims (end) %

% !TEX root = ./main.tex

\section{Problem Formulation: NN Architectures for Control} % (fold)
\label{sec:problem_statement}

We begin by stating the first main problem that we will consider in this  
paper: that of designing an assured NN architecture for nonlinear control  
(hereafter referred to as the \textbf{controller architecture problem}).  
Specifically, we wish to identify a ReLU architecture to be used as a feedback 
controller for the control system $\Sigma$; this architecture must further come 
with the \emph{assurance} that there exist parameter weights for which the 
realized NN controller controls $\Sigma$ to some proscribed specification.
% specification that can be met by some other, non-NN controller.

However, for pedagogical reasons, we will state two versions of the controller 
architecture problem in this section. The first will be somewhat generic in 
order to motivate a crucial innovation of this paper: a new simulation 
relation, Abstract Disturbance Simulation (ADS) (see also 
\cite{FerlezTwoLevelLatticeNeural2020a}). The second formulation of this 
problem, then, actually incorporates ADS into a formal problem statement, where 
it serves to facilitate the design of assured controller architectures. Our 
solution of this second, more specific version, is the main contribution of 
this paper, and  appears as \cref{thm:main_theorem} in 
\cref{sec:a_relu_architecture_for_nonlinear_systems}.

% In the second problem, we wish to 
% identify an architecture for a ReLU network to used as controlled vector field 
% \emph{instead} of $\Sigma$: this architecture must have parameter weights that 
% allow it to approximate the vector field of $\Sigma$ sufficiently well that the 
% trained ReLU, $\nn$, can be used in the control system $\Sigma$ instead of the 
% original vector field, $f$.
%\\[-5pt]

%\noindent\textbf{Main Problem}\\[-5pt]
\subsection{Generic Controller Architecture Problem} % (fold)
\label{sub:generic_control_architecture_problem}

As noted in \cref{sec:introduction}, designing an \emph{assured} NN 
architecture for control hinges on the well-posedness of the desired (binary) 
closed-loop specification, and this is as much a statement about the 
specification as it is about the architecture. Thus, a formal problem of NN 
architecture design (for control) necessarily begins with a framework for 
describing closed-loop system specifications.

% both for a particular NN architecture but also in general -- i.e. for 
% \emph{any} controller.

To this end, we will formulate our controller architecture problem in terms of 
a $\tau$-sampled metric transition system embedding of the underlying 
continuous-time models (see \cref{sub:dynamical_model}). Although this choice 
may seem an unnatural deviation from the underlying continuous-time models, it 
affords two important benefits. First, metric transition systems come with a 
natural and flexible notion of specification satisfaction in the form of 
(bi)simulation relations. In this paradigm, specifications are described by 
means of \emph{another} transition system that encodes the specification; the 
original system then satisfies the specification if it is simulated by the 
(transition system) encoding of the specification. Importantly, it is well 
known that a diverse array of specifications can be captured in this context, 
among them LTL formula satisfaction \cite{TabuadaVerificationControlHybrid2009} 
and stability. Secondly, the sample period $\tau$ constitutes an additional 
degree of freedom in the specification relative to the original continuous-time 
system (or a proscribed fixed sample-period embedding); this extra degree of 
freedom will facilitate the development of assured  NN architectures.

From this, we consider the following generic formulation of a controller 
architecture design problem.
\begin{problem}[Controller Architecture Design -- Generic Formulation]
\label{prob:main_problem_generic}
	Let $\tau > 0$ and $K_\text{cont} \negthinspace > \negthinspace 0$ be 
	given. Let $\Sigma$ be a feedback controllable Lipschitz control system, 
	and let $S_\text{spec} \negthinspace = \negthinspace 
	(X_\text{spec},U_\text{spec}, 
	\ltsxition{S_\text{spec}\negthickspace\negthickspace}{~})$ be a transition 
	system encoding for a specification on $\Sigma$.
	% Finally, let $\tau = 
	% \tau(K_x,K_u,\mathcal{K},K_\text{cont})$ be determined by the parameters 
	% specified.

	Now, suppose that there exists a %$\delta$,$\tau$ positively invariant 
	Lipschitz-continuous controller $\Psi: \mathbb{R}^n \rightarrow U$ with 
	Lipschitz constant $K_\Psi \le K_\text{cont}$ s.t.:
	\begin{equation}
	\label{eq:main_prob_spec_assumtion_generic}
		S_\tau(\Sigma_\Psi)
			\preceq
		S_\text{spec}.
	\end{equation}

	\noindent Then the problem is to find a ReLU architecture, 
	$\text{Arch}(\Theta)$, with the property that there exists values for 
	$\Theta$ such that:
	\begin{equation}\label{eq:main_prob_spec_conclusion_generic}
		S_{\tau}(\Sigma_{\negthinspace\nn\negthinspace_\Theta})
			\preceq
		S_\text{spec}.
	\end{equation}
	In both \eqref{eq:main_prob_spec_assumtion_generic} and 
	\eqref{eq:main_prob_spec_conclusion_generic}, $\precsim$ is as defined in 
	% simulation relation as in 
	\cref{def:simulation_relation}.
\end{problem}

% For the purposes of the problem statement, the abstract disturbance relation 
% $\preceq^\epsilon_{\mathcal{AD}_\delta}$ should be interpreted as enforcing 
% specification satisfaction just as $\epsilon$-approximate simulation does in 
% \eqref{eq:main_prob_spec_conclusion}. In fact, we will show later 
% (\cref{prop:ads_implies_epsilon_approx}) that abstract disturbance simulation 
% implies $\epsilon$-approximate simulation (for a common $\epsilon$), so it 
% represents a notion of specification satisfaction that is at least as strong as 
% $\epsilon$-approximate simulation. However, $\epsilon$-approximate simulation 
% is strictly weaker, and it is by itself inadequate to frame the assumption in 
% \eqref{eq:main_prob_spec_assumtion}. In 
% \cref{sub:abstract_disturbance_simulation}, we will show that 
% \cref{prob:main_problem_generic}, stated using $\epsilon$-approximate 
% simulation in \eqref{eq:main_prob_spec_assumtion} instead, has no solution for 
% some choices of $\Sigma$, $S_\text{spec}$ and $\Psi$.

The main assumption in \cref{prob:main_problem_generic} is that there exists a 
controller $\Psi$ which satisfies the specification, $S_\text{spec}$. We use 
this assumption primarily to help ensure that the problem is well posed:  
indeed, it is known that there are nonlinear control problems for which 
\emph{no} continuous controllers exits. Thus, this assumption is in some sense 
an essential requirement to formulate a well-posed controller architecture 
problem: for if there exists no such $\Psi$, then there could be no NN 
controller that satisfies the specification, either, since the latter also  
belongs to the class of Lipschitz continuous functions (modulo discrepancies  
in Lipschitz constants). 
% In other words, this assumption ensures that we aren't trying to assert the 
% existence of NN controller for a system and specification that can't be 
% achieved by \emph{any} Lipschitz-continuous controller -- such examples are 
% known to exist for nonlinear systems.
In this way, the existence of a controller $\Psi$ also subsumes any possible 
conditions on the nonlinear system that one might wish to impose: 
stabilizability, %or controllability
for example.
% Finally, note that this assumption is actually used in a relatively 
% weak way in \cref{prob:main_problem_generic}: beyond the  \emph{existence} of 
% $\Psi$, the NN architecture design doesn't depend on explicit knowledge of $f$ 
% beyond the constants $K_x$, $K_u$ and $\mathcal{K}$ (of course verifying the 
% existence of such a $\Psi$ may require full knowledge of $f$).

That \cref{prob:main_problem_generic} is more or less ill-posed without 
assuming the existence of a controller $\Psi$ (for \emph{some} constant 
$K_\text{cont}$) suggests a natural solution to the problem. In particular, a 
NN architecture can be bootstrapped from this knowledge by simply designing an 
architecture that is sufficiently parameterized as to adequately approximate 
\emph{any} such $\Psi$, i.e. any function of Lipschitz constant at most 
$K_\text{cont}$ (this is a preview of the approach we will subsequently use). 
Unfortunately, however, this approach also reveals a deficiency in the 
assumption associated with $\Psi$: controller approximation necessarily 
introduces instantaneous control errors relative to $\Psi$, and these errors 
can compound transition upon transition from the dynamics. As a consequence, 
the assumed information about $\Psi$ is actually not as immediately helpful as 
it appears. \emph{In particular, if $\Psi$ is not very robust, then the 
accumulation of such errors could make it impossible to prove that an 
(approximate) NN controller satisfies the same specification as $\Psi$, viz. 
\eqref{eq:main_prob_spec_conclusion_generic}. }

This effect can be seen directly in terms of the simulation relations in 
\eqref{eq:main_prob_spec_assumtion_generic} and 
\eqref{eq:main_prob_spec_conclusion_generic}. Take $x \in X_\tau$ and consider 
two transitions from $x$: one in $S_\tau(\Sigma_\Psi)$ given by $x  
\xition{\Psi}{~} x^\prime$ and one in 
$S_{\tau}(\Sigma_{\negthinspace\nn\negthinspace_\Theta})$ given by $x  
\xition{{\negthinspace {\nn_\Theta}}}{~} x^{\prime\prime}$. Note that if 
$\nn_\Theta$ merely approximates $\Psi$, there will in general be a 
discrepancy between $x^\prime$ and $x^{\prime\prime}$, i.e. $x^\prime \neq  
x^{\prime\prime}$. Thus, although $x^\prime$ necessarily has a simulating state 
in $S_\text{spec}$ by assumption, $x^{\prime\prime}$ need not have its own 
simulating state in $S_\text{spec}$. This follows because the simulation 
relation in \cref{def:simulation_relation} can assert simulating states only 
through transitions (i.e. \eqref{eq:ordinary_simulation}), and there may be no 
transition $x \xition{\Psi}{~} x^{\prime\prime}$ in $S_\tau(\Sigma_\Psi)$.

\subsection{Abstract Disturbance Simulation} % (fold)
\label{sub:perturbation_}

Motivated by the observations above, we propose a new simulation relation as a 
formal notion of specification satisfaction for metric transition  systems; we 
call this relation \emph{abstract disturbance simulation} or ADS (see also 
\cite{FerlezTwoLevelLatticeNeural2020a}). Simulation by means of an ADS 
relation is stronger than ordinary simulation   
(\cref{def:simulation_relation}) in order to incorporate a notion of 
robustness. Thus, abstract disturbance simulation is inspired by -- and is 
related to -- both robust bisimulation 
\cite{KurtzRobustApproximateSimulation2020} and especially disturbance 
bisimulation \cite{MallikCompositionalSynthesisFiniteState2019}. Crucially 
however, it abstracts those notions away from their definitions in terms of 
specific control system embeddings and explicit modeling of disturbance inputs. 
As a result, ADS can then be used in a generic context such as the one 
suggested by \cref{prob:main_problem_generic}.

Fundamentally, ADS still functions in terms of conventional simulation 
relations; however, it incorporates robustness by first augmenting the 
simulated system with ``virtual'' transitions, each of which has a target that 
is perturbed from the target of a corresponding ``real'' transition. In this 
way, it is conceptually similar to the technique used in 
\cite{ZamaniSymbolicModelsNonlinear2012} and 
\cite{PolaApproximatelyBisimilarSymbolic2008} to define a quantized 
abstraction, where deliberate non-determinism is introduced in order to account 
for input errors. As a result of these additional transitions, when a metric 
transition system is ADS simulated by a specification, this implies that the 
system robustly satisfies the specification relative to satisfaction merely by 
an ordinary simulation relation.
%Abstract disturbance simulation 
% enforces a notion of specification that is robust to perturbation of the 
% state, and this will facilitate solving the main problem in this paper.
%
%

As a prerequisite for defining ADS, we introduce the following definition: it 
captures the idea of augmenting a metric transition system with virtual, 
perturbed transitions.
\begin{definition}[Perturbed Metric Transition System]
\label{def:perturbed_system}
	Let $S = (X, U, \ltsxition{S}{~})$ be a metric transition system where $X 
	\subseteq X_M$ for a metric space $(X_M, d)$. Then the 
	\textbf{$\delta$-perturbed metric transition system} of $S$, 
	$\mathfrak{S}^\delta$, is a tuple $\mathfrak{S}^\delta = (X, U, 
	\ltsxition{\mathfrak{S}\negthinspace\overset{\delta}{~} 
	\thickspace\thinspace}{~})$ where the (altered) transition relation, 
	$\ltsxition{\mathfrak{S}\negthinspace\overset{\delta}{~} 
	\thickspace\thinspace}{~}$, is defined as:% follows:
	\begin{equation}
		x \hspace{-0.5mm} 
		\ltsxition{\mathfrak{S}\negthinspace\overset{\delta}{~} \thickspace\thinspace}{u}
		x^\prime
		\text{ iff } \;
		\exists x^{\prime\prime} \hspace{-0.8mm} \in \hspace{-0.5mm}  X \text{ s.t. } d(\hspace{-0.2mm} x^{\prime\prime} \hspace{-0.2mm} , \hspace{-0.2mm} x^\prime \hspace{-0.2mm}) \hspace{-0.4mm} \leq  \hspace{-0.3mm} \delta \text{ and } x \ltsxition{S}{u} x^{\prime\prime}\hspace{-0.8mm}.
	\end{equation}
\end{definition}
Note that $\mathfrak{S}^\delta$ has identical states and input labels to $S$, 
and it also subsumes all of the transitions therein, i.e. $\ltsxition{S}{~} 
\thinspace \subset \thinspace 
\ltsxition{\mathfrak{S}\negthinspace\overset{\delta}{~} 
\thickspace\thinspace}{~}$. However, as noted above, the transition relation 
for $\mathfrak{S}^\delta$ explicitly contains new nondeterminism relative to 
the transition relation of $S$; each additional nondeterministic transition 
is obtained by perturbing the target state of a transition in $S$.
%This nondeterminism can be thought of in terms 
%of perturbations on the target state of each transition in $S$; each such 
%perturbation becomes the target of a (nondeterministic) transition with the 
%same input label as the original transition.

%With this definition in hand,
Using this definition, we can finally define an 
abstract disturbance simulation between two metric transition systems.
\begin{definition}[Abstract Disturbance Simulation]
\label{def:delta_perturbation_simulation} 
	Let $S = (X_S, U, \ltsxition{S}{~})$ and $T = (X_T, U_T, \ltsxition{T}{~})$ 
	be metric transition systems whose state spaces $X_S$ and $X_T$ are subsets 
	of the same metric space $(X_M, d)$. Then $T$ \textbf{abstract-disturbance 
	simulates} $S$ under disturbance $\delta$, written $S 
	\preceq_{{\mathcal{AD}_\delta}} T$ if there is a relation $R \subseteq X_S 
	\times X_T$ such that

	\begin{enumerate}
		\item for every $(x,y) \in R$, $d(x,y) \leq \delta$;

		\item for every $x \in X_S$ there exists a pair $(x,y) \in R$; and

		\item for every $(x,y) \in R$ and $x 
			\ltsxition{\mathfrak{S}\negthinspace\overset{\delta}{~} 
			\thickspace\thinspace}{u} x^\prime$ there exists a $y 
			\ltsxition{T}{v} y^\prime$ such that $(x^\prime, y^\prime) \in R$.
	\end{enumerate}
\end{definition}

\begin{remark}
	$\preceq_{\mathcal{AD}_0}$ corresponds with the usual notion of simulation 
	for metric transition systems. Thus,
	\begin{equation}
		S \preceq_{\mathcal{AD}_\delta} T \Leftrightarrow \mathfrak{S}^\delta \preceq_{\mathcal{AD}_0} T.
	\end{equation}
\end{remark}

% subsection perturbation_ (end)

\subsection{Main Controller Architecture Problem} % (fold)
\label{sub:main_controller_architecture_problem}

Using ADS for specification satisfaction, we can now state the version of 
\cref{prob:main_problem_generic} that we will consider and solve as the main 
result of this paper.

\begin{problem}[Controller Architecture Design]
\label{prob:main_problem}
	Let $\delta \negthinspace > \negthinspace 0$, $\tau > 0$ and $K_\text{cont} 
	\negthinspace > \negthinspace 0$ be given. Let $\Sigma$ be a feedback 
	controllable Lipschitz control system, and let $S_\text{spec} \negthinspace 
	= \negthinspace (X_\text{spec},U_\text{spec}, 
	\ltsxition{S_\text{spec}\negthickspace\negthickspace}{~})$ be a transition 
	system encoding for a specification on $\Sigma$.
	% Finally, let $\tau = 
	% \tau(K_x,K_u,\mathcal{K},K_\text{cont},\delta)$ be determined by the 
	% parameters specified.

	Now, suppose that there exists a $\delta$,$\tau$ positively invariant 
	Lipschitz-continuous controller $\Psi: \mathbb{R}^n \rightarrow U$ with 
	Lipschitz constant $K_\Psi \le K_\text{cont}$ such that:
	\begin{equation}
	\label{eq:main_prob_spec_assumtion}
		S_\tau(\Sigma_\Psi)
			\preceq_{\mathcal{AD}_\delta}
		S_\text{spec}.
	\end{equation}

	\noindent Then the problem is to find a ReLU architecture, 
	$\text{Arch}(\Theta)$, with the property that there exists values for 
	$\Theta$ such that:
	\begin{equation}\label{eq:main_prob_spec_conclusion}
		S_{\tau}(\Sigma_{\negthinspace\nn\negthinspace_\Theta})
			\preceq_{\mathcal{AD}_0}
		S_\text{spec}.
	\end{equation}
\end{problem}

\cref{prob:main_problem} is distinct from \cref{prob:main_problem_generic} in 
two crucial ways. The first of these is the foreshadowed use of ADS for 
specification satisfaction with respect to $\Psi$. In particular, we now assume 
the existence of a controller, $\Psi$, that satisfies the specification up to 
some fixed robustness margin $\delta$, as captured by 
$\preceq_{\mathcal{AD}_\delta}$. This will be the main technical facilitator of 
our solution, since it enables the design of a NN architecture around the 
(still unknown) controller $\Psi$.
% As a part of this alteration, the sample 
% period $\tau$ is also now allowed to depend on the robustness margin $\delta$ 
% in addition to the other relevant constants in \cref{prob:main_problem_generic}.

However, \cref{prob:main_problem} also has a second additional assumption 
relative to \cref{prob:main_problem_generic}: the controller $\Psi$ must also 
be  $\delta$-$\tau$ positive invariant with respect to the compact subset of 
states under consideration, $X$ -- see 
\cref{def:delta_tau_positive_invariance}. This is a technically, but not 
conceptually, relevant assumption, and it is an artifact of the fact that we 
are confining ourselves to a compact subset of the state space. It merely 
ensures that those $\delta$ perturbations created internally to 
$\preceq_{\mathcal{AD}_\delta}$ will lie entirely within $X$. $\delta$-$\tau$ 
invariance captures this by asserting that those states within $\delta$ of the 
boundary of $X$ (i.e. $\text{edge}_\delta(X)$) are ``pushed'' sufficiently 
strongly towards the interior of $X$ so that after $\tau$ seconds, they are no 
longer $\delta$ close to the boundary of $X$ (i.e. in $\text{edge}_\delta(X)$). 
Thus, every transition in $S_\tau(\Sigma_\Psi)$ starting from 
$\text{edge}_\delta(X)$ has a target in $X \backslash  \text{edge}_\delta(X)$, 
and the $\delta$-perturbed version of $S_\tau(\Sigma_\Psi)$ has no transitions 
with targets outside of $X$.

% Moreover, there is a strong conceptual reason to consider abstract disturbance 
% simulation in specification satisfaction for such a $\Psi$. Our approach to 
% solve this problem will be to design a NN architecture that can approximate 
% \textbf{\itshape any} such $\Psi$ sufficiently closely. However, $\nn_\Theta$ 
% clearly belongs to a smaller class of functions than $\Psi$, so an arbitrary 
% controller $\Psi$ cannot, in general, be represented \emph{exactly} by means of 
% $\nn_\Theta$. This presents an obvious difficulty because instantaneous errors 
% between $\Psi$ and $\nn_\Theta$ may accumulate by means of the system dynamics, 
% i.e. via \eqref{eq:feedback_integral_eq}.

% subsection main_controller_architecture_problem (end)

% section problem_formulation (end)

 %

%\input{symbolic} %

% !TEX root = ./main.tex

\section{ReLU Architectures for Controlling Nonlinear Systems} % (fold)
\label{sec:a_relu_architecture_for_nonlinear_systems}

We are almost able to state main theorem of this paper: that is 
\cref{thm:main_theorem}, which directly solves \cref{prob:main_problem}. As a 
necessary prelude, though, we introduce the following two definitions. The 
first formalizes the distance between coordinate-wise upper and lower bounds of 
a compact set (\cref{def:extent}). The second formalizes the notion of a 
(rectangular) grid of points that is sufficiently fine to cover a compact set 
$X$ by $\sup$-norm balls of a fixed size (\cref{def:eta_grid}).
% \begin{definition}[Vector Field Bound, $\mathcal{K}$]
% \label{def:vector_field_bound}
% 	Let:
% 	\begin{equation}
% 		\mathcal{K} \triangleq \max_{x\in X, u \in U} \lVert f(x,u) \rVert,
% 	\end{equation}
% 	which is well defined because $X \times U$ is compact and $f$ is continuous.
% \end{definition}

\begin{definition}[Extent of $X$]
\label{def:extent}
	The \textbf{extent} of a compact set $X$ is defined as:
	\begin{equation}
		\text{ext}(X) \triangleq \max_{k=1,\dots,n} \left| 
		\max_{x \in X} \pi_k(x) - \min_{x\in X}\pi_k(x)
		\right|.
	\end{equation}
	% where $\pi_k(x)$ is the projection of $x$ onto its $k^\text{th}$ component.
\end{definition}
Indeed, the extent of a compact set may also be regarded as the smallest edge 
length of a hypercube that can contain $X$.
\begin{definition}[$\eta$-grid]\label{def:eta_grid}
	Let $\eta > 0$ be given, and let $X$ be compact and connected with  
	non-empty interior. Then a set $X_\eta \subset X$ is an $\eta$-grid of $X$ 
	if
	\begin{itemize}
		\item the set of $\eta/2$-balls, $X_\text{part} \triangleq  
			\{B(\mathbf{x};\eta/2) | \mathbf{x} \in X_\eta\}$, has the 
			properties that:

			\begin{enumerate}
				\item for all $\mathbf{x}, \mathbf{x}^\prime \in  X_\eta$ 
					and $i = 1, \dots, n$, there is an integer $k_i \in 
					\mathbb{Z}$ such that $\pi_i(\mathbf{x}) = \pi_i(\mathbf{x}^\prime) + k_i 
					\cdot \eta$; and 

				\item $X \subseteq \bigcup_{\mathbf{x} \in X_\eta}  
					\overline{B}(\mathbf{x};\eta)$.
				% \item for all $A, B \in X_\text{part}$, $A\cap B =  \text{edge}_0(A) 
				% 	\cap \text{edge}_0(B)$
			\end{enumerate}
	\end{itemize}
	Elements of $X_\eta$ will be denoted by bold-face font, i.e.  $\mathbf{x} 
	\in X_\eta$.
	%\cref{fig:extra_corner_definition} is a cartoon  that depicts 
	%an $\eta$-grid, among other things.
	% the set of centers 
	% of a regular grid of $\eta/2$ $\sup$-norm balls that partitions $X$; such 
	% an $\eta$-grid will be denoted by $X_\eta \subset X$.

	% That is let $X_\eta \subset X$ and define  then $X_\eta$ is an $\eta$-grid if it 
	% has the following two  properties:
	
	% Finally, suppose that $X_\eta$ has a canonical enumeration denoted by 
	% $X_\eta = \{x_i\}_{i=0}^{|X_\eta|-1}$.
\end{definition}
Note: \emph{1)} asserts that the elements of $X_\eta$ are spaced on a 
rectangular grid, and \emph{2)} asserts that centering a closed ball at each 
element of $X_\eta$ covers the set $X$. 
\begin{remark}
	It is not the case that a compact and connected set with non-empty interior 
	necessarily has an $\eta$ grid for any arbitrary choice of $\eta$; see 
	\cref{fig:not_eta_grid} for an illustration.
\end{remark}
\begin{figure}[h]
	\centering
	% \vspace{1.4mm}
	\includegraphics[width=0.46\textwidth, trim={0 0 0 0},clip]{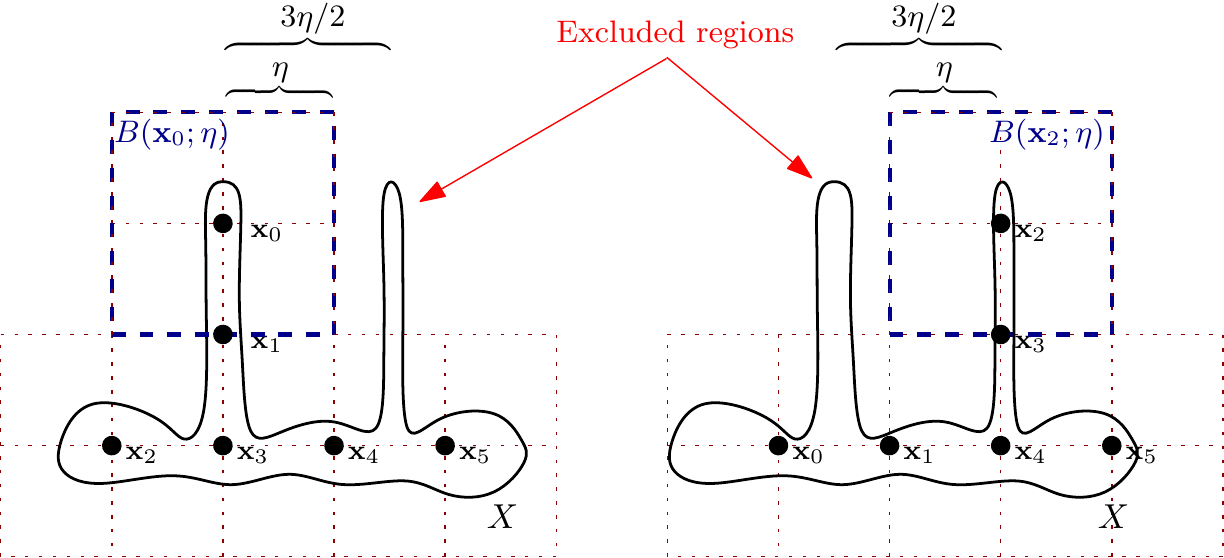} %
	\caption{Illustration of a set $X$ and a choice of $\eta$ for which no $\eta$-grid exists: any attempt to cover one ``spike'' with grid points will result in excluding the other. (Red dashed lines indicate $\eta/2$ balls, and sold circles indicate elements of $X_\eta$.)} 
	\label{fig:not_eta_grid} %\vspace{-2.5mm}
\end{figure}

Now we can state the main theorem of the paper.
\begin{theorem}[ReLU Architecture for Control]\label{thm:main_theorem}
	Let $\delta > 0$, $\tau > 0$ and $K_\text{cont} > 0$ be given, and let 
	$\Sigma$ and $S_\text{spec}$ be as in the statement of 
	\cref{prob:main_problem}.  Furthermore, suppose that there exists a 
	$\delta,\tau$ positively invariant Lipschitz continuous controller $\Psi: 
	\mathbb{R}^n \rightarrow U$ with Lipschitz constant $K_\Psi \leq 
	K_\text{cont}$ such that: 
	\begin{equation}
	\label{eq:main_thm_spec_assumtion}
		S_\tau(\Sigma_\Psi)
			\preceq_{\mathcal{AD}_\delta}
		S_\text{spec}.
	\end{equation} 
	Finally, suppose that $\mu > 0$ such that:
	\begin{equation}\label{eq:main_inequality}
		K_u \cdot \mu \cdot \tau \cdot
		e^{(K_x + 2 K_u  K_\text{cont}) \tau } < \delta.
	\end{equation}

	If $\eta \leq \frac{\mu}{{3} \cdot K_\text{cont}}$ is such that there 
	exists an $\eta$-grid of $X$, then there exists an $m$-fold parallel TLL NN 
	architecture $\text{Arch}(\tllThetaPar_{^n_m\negthinspace N})$ of size  
	(see \cref{cor:parallel_tll}): 
	\begin{equation}\label{eq:main_controller_arch_bound}
		N \geq 
		% \left(
		% 	n! \cdot \sum_{k=1}^{n} \frac{2^{2k-1}}{(n-k)!}
		% \right)
		n!
		\cdot
		\left\lceil\frac{\text{ext}(X)}{\eta} + 2\right\rceil^n
	\end{equation}
	with the property that there exist values for 
	$\tllThetaPar_{^n_m\negthinspace N}$ such that:
	\begin{equation}\label{eq:main_thm_spec_conclusion}
		S_{\tau}(\Sigma_{\negthinspace{\scriptstyle \nn\atop ~}\negthinspace {\tllThetaPar_{^n_m\negthinspace N}}})
			\preceq_{\mathcal{AD}_0}
		S_\text{spec}.
	\end{equation}
\end{theorem}

\begin{remark}\label{rem:improved_bound}
	Note that the coefficient $n!$ in \eqref{eq:main_controller_arch_bound} is 
	significantly smaller than the analogous one exhibited in 
	\cite{FerlezTwoLevelLatticeNeural2020a}: viz. $\left(n! \cdot 
	\sum_{k=1}^{n} \tfrac{2^{2k-1}}{(n-k)!}\right)$. Furthermore, 
	\cite{FerlezTwoLevelLatticeNeural2020a} also failed  to account for the 
	need to choose an $\eta$-grid.
\end{remark}

The size of the architecture specified by \cref{thm:main_theorem} is 
effectively determined by the ratio $\text{ext}(X)/\eta$ (for given state  and 
control dimensions). This quantity has two specific connections to the 
assumptions in \cref{prob:main_problem}. On the one hand, the maximum allowable 
$\eta$ is set by $\mu$, which is determined by the Lipschitz constants of the 
dynamics and the properties of the assumed controller, $\Psi$. On the other 
hand, the specific character of state set itself influences the size the 
architecture, both explicitly by way of $\text{ext}(X)$ and implicitly via the 
requirement that an $\eta$-grid exists for $X$.

With regard to the first of these influences, \eqref{eq:main_inequality} can be 
rearranged to provide the following upper bound for $\eta$:
\begin{equation}
	\eta \leq \tfrac{1}{{3} \cdot K_\text{cont} \cdot K_u} e^{-(K_x + 2 K_u K_\text{cont})\tau} \cdot \left( \tfrac{\delta}{\tau} \right).
\end{equation}
Of unique interest is the influence of the unknown but asserted-to-exist 
controller, $\Psi$. In particular, note that the \emph{robustness} of $\Psi$ 
has an intuitive effect on the size of the architecture. As the robustness of 
the asserted controller increases (for fixed $\tau$), i.e. as $\delta$  
increases, $\eta$ is permitted to be larger, and the architecture can be 
correspondingly smaller. Of course as the asserted controller becomes  
\emph{less} robust, i.e. as $\delta$ decreases, the architecture must 
correspondingly become larger. This trade off makes intuitive sense in light of 
our eventual proof strategy: in effect, we will design an architecture that can 
approximate \emph{any} potential $\Psi$, so the more robust $\Psi$ is, the less 
precisely the NN architecture must be able to approximate it -- so a smaller 
architecture suffices (and vice versa). See 
\cref{sub:proof_sketch_of_main_theorem}.% for more details.

With regard to the second influence, the extent of the set $X$ has a  
straightforward and direct effect on the designed architecture: the  ``larger'' 
the set is, the larger the architecture is required. However, the 
``complexity'' of the set $X$ indirectly influences the size of the 
architecture via the $\eta$-grid requirement. Indeed, if the boundary of $X$ 
has many thin protuberances (see \cref{fig:not_eta_grid}, for example), then an 
extremely fine $\eta$ grid may be required to ensure that grid points are 
placed throughout $X$. Crucially, this choice of $\eta$ may need to be 
significantly smaller than the maximum allowable $\eta$ computed via 
\eqref{eq:main_inequality}, as above -- the result will be a significantly  
larger architecture than suggested by \eqref{eq:main_inequality} alone. 
Unfortunately, this effect is difficult to quantify directly, given the 
variability in complexity of state sets. Nevertheless, given any connected 
compact set $X$, there exists an $\eta > 0$ for which $X$ has an $\eta$ grid.

\begin{proposition}
	Let $X \subset \mathbb{R}^n$ be compact and connected with a non-empty  
	interior. Then there exists an $\eta > 0$ such that there exists an $\eta$ 
	grid of $X$.
\end{proposition}
\begin{proof}
	We will consider dyadic $\eta$ grids: that is grids based on $\eta_k =  
	2^{-k}$, where the associated candidate grid is given by
	\begin{equation}
		X_{\eta_k} \triangleq \{ x \in X \; | \;\exists z \in \mathbb{Z} . x = \tfrac{z}{2^k} \}.
	\end{equation}
	For convenience, define the following notation:
	\begin{equation}
		H_{\eta_k} \triangleq \cup_{\mathbf{x}\in X_{\eta_k}} \overline{B}(\mathbf{x}; \eta_k)
	\end{equation}

	Now because $X$ is connected with non-empty interior, it the case that 
	$\text{bd}(X) = \overline{\text{int}(X)}$. Hence, for every $x \in X$, 
	there exists an  $K_x$ such that for all $k \geq K_x$, $x \in H_{\eta_k}$ 
	(consider a truncation of the binary expansion of each coordinate of $x$). 
	It follows that if there is a $K$ such that $K_x < K$ for all $x \in X$, 
	then the claim is proved  (simply choose $X_{\eta_{K}}$).

	Thus, suppose by contradiction that there is a divergent sequence $\{x_k\}$ 
	s.t. $K_{x_k} < K_{x_{k+1}}$ for all $k \in \mathbb{N}$. However, $X 
	\subset  \mathbb{R}^n$ is compact, so $\{x_k\}$ has a convergent 
	subsequence $\{x_{k_\ell}\}$ with limit $x^\prime \in X$; this subsequence 
	retains the property that $K_{x_{k_\ell}} < K_{x_{k_{\ell+1}}}$. Moreover, 
	by the above claim, there is some $K_{x^\prime}$ such that $x^\prime \in 
	H_{\eta_{k}}$ for all $k \geq K_{x^\prime}$.

	Now we consider two cases: first that $x^\prime \in B(\mathbf{x};  
	\eta_{K_{x^\prime}})$ for some $\mathbf{x}$, and second that $x^\prime$ is 
	on the boundary of some $B(\mathbf{x}; \eta_{K_{x^\prime}})$. In the first 
	case, we note that there exists an $L$ such that for all $\ell \geq L$, 
	$x_{k_\ell} \in B(\mathbf{x};  \eta_{K_{x^\prime}})$, simply by the 
	convergence of the subsequence to  $x^\prime$. This is clearly a 
	contradiction, since it implies that $x_{k_\ell} \in 
	H_{\eta_{K_{x^\prime}}}$ for all $\ell \geq L$.

	On the other hand, suppose $x^\prime$ is on the boundary of some 
	$B(\mathbf{x}; \eta_{K_{x^\prime}})$, and suppose that $x^\prime$ belongs 
	to a $d$-dimensional face of $B(\mathbf{x}; \eta_{K_{x^\prime}})$ where $1  
	\leq d < n$ (if $d = 0$ then $x^\prime$ is itself a dyadic point, and the 
	above argument applies directly, since $x^\prime \in  
	X_{\eta_{K_{x^\prime}}}$).
	% Then for every 
	% $K > K_{x^\prime}$, $x^\prime$ is part of a $d$-dimensional face, $f_K$, of 
	% some $B(\mathbf{x}; \eta_{K})$ with $\mathbf{x} \in X_{\eta_K}$ (finer 
	% dyadic $\eta$ grids create closed balls whose faces are contained in faces 
	% of closed balls from coarser $\eta$ grids). 
	Let this face be denoted $F_{x^\prime}$. Then there exists a finite $K_0 > 
	K_{x^\prime}$ such that $x^\prime \in B(\mathbf{x}; \eta_{K_0})$ for some 
	$\mathbf{x} = z/2^{K_0}$: use coordinate-wise binary expansions to find a 
	dyadic point within the face $F_{x^\prime}$ that includes 
	$x^\prime$ in the associated open $\eta_K$ ball (this is  possible since 
	$F_{x^\prime}$ has at least one non-dyadic  coordinate by the $d 
	\geq 1$ assumption). This leads to the same contradiction as before, since 
	a tail of the subsequence $\{x_{k_\ell}\}$ is eventually contained in this 
	ball.
	% Thus, we proceed by contradiction: suppose no $X_{\eta_k}$ is an $\eta$ 
	% grid of $X$. That is suppose for all $k \in \mathbb{N}$ there exists an 
	% $x_k$ such that
	% \begin{equation}
	% 	x_k \notin \cup_{\mathbf{x} \in X_{\eta_k}} B(\mathbf{x}; \eta_k/2).
	% \end{equation}
	% However, the sequence $\{x_k\}$ is in a compact subset of $\mathbb{R}^n$, 
	% and hence has a convergent subsequence $\{x_{k_\ell}\}$ with limit 
	% $x^\prime \in X$. This  contradicts the above claim, which asserts that 
	% there exists a $K_{x^\prime}$ s.t. $x^\prime \in X_{\eta_k}$ for all $k 
	% \geq K_{x^\prime}$. This is because there exists a $L$ such that for all 
	% $\ell \geq L$, $\lVert x_{k_\ell} - x^\prime \rVert < 2^{-K_{x^\prime}}$.  
	% Choosing $l \geq L$
\end{proof}

The remainder of this section is divided as follows. 
\cref{sub:proof_sketch_of_main_theorem} contains a proof sketch of 
\cref{thm:main_theorem}, which divides the proof into two main  intermediate 
steps. \cref{sec:thm1_part1} and \cref{sec:thm1_part2} thus contain the formal 
proofs for these intermediate steps. The overall formal proof of 
\cref{thm:main_theorem} then appears in \cref{sec:proof_of_main_theorem}. 

% \noindent\textbf{Proof Sketch of \cref{thm:main_theorem}:} %\\%[-5pt]

\subsection{Proof Sketch of \cref{thm:main_theorem}} % (fold)
\label{sub:proof_sketch_of_main_theorem}

Our proof of \cref{thm:main_theorem} implements the following simple strategy. 
By assumption, a $\Psi$ exists that satisfies the specification robustly (via 
ADS), and hence, we show that any suitably close approximation of $\Psi$ will  
function as a controller that satisfies the specification as well (albeit 
non-robustly). This implies that we need only design a NN architecture with 
enough parameterization that it can approximate \emph{any} possible $\Psi$ that 
satisfies the conditions of the theorem.

There is, however, an important and non-obvious sequence of observations 
required to design such a NN architecture. To start, any such $\Psi$ is 
Lipschitz continuous, so it is possible to uniformly approximate it by 
interpolating between its values taken on a uniform grid. Moreover, its 
Lipschitz constant %of any such $\Psi$ 
has a known upper bound, so for a given approximation accuracy, the fineness of 
that grid can be chosen conservatively, and hence \emph{independent} of any 
particular $\Psi$. However, we show it is possible to interpolate between 
points over such a uniform grid using a CPWA that has a number of linear 
regions (\cref{def:linear_region}) proportional to the number of points in the 
grid -- which we reiterate is independent of any particular $\Psi$. This CPWA 
is in effect \emph{parameterized} by the values it takes on those grid points, 
but in such a way that its number of linear regions is \emph{independent} of 
those parameter values. This shows that an arbitrary $\Psi$ can be approximated 
by a CPWA with a fixed number of linear regions; it remains to connect this to 
a NN architecture. Fortunately, the TLL NN architecture 
\cite{FerlezAReNAssuredReLU2020} can be used directly for this purpose by way 
of \cref{thm:tll_architecture} \cite[Theorem  7]{FerlezAReNAssuredReLU2020}: 
the result in question explicitly specifies a NN architecture that can 
implement any CPWA with a known, bounded number of linear regions.

Consequently, the proof of \cref{thm:main_theorem} can instead be decomposed 
into establishing the following two implications:

\begin{enumerate}[{Step }1)]
	\item \emph{``Approximate controllers satisfy the specification'':} 
		There is a approximation accuracy, $\mu$, and sampling period, $\tau$, 
		with the following property: if the unknown controller $\Psi$ satisfies 
		the specification (under $\delta$ disturbance and sampling period 
		$\tau$), then any controller -- NN or otherwise -- that approximates 
		$\Psi$ to accuracy $\mu$ in $\lVert \cdot \rVert_X$ will also 
		satisfy the specification (but under no disturbance). %This implication is shown in~
		See \cref{lem:final_spec}. % of \cref{sec:thm1_part1}. 

	\item \emph{``Any controller can be approximated by a CPWA with the same 
		fixed number of linear regions'':} If unknown controller $\Psi$ has a 
		Lipschitz constant $K_\Psi \leq K_\text{cont}$, then $\Psi$ can be 
		approximated by a CPWA with a number of regions that depends only on 
		$K_\text{cont}$ and the approximation accuracy.
		%This implication 
		%is shown in~
		See \cref{cor:interpolation_corollary}.% of \cref{sec:thm1_part2}. 
\end{enumerate}
%
% We will show these results for \emph{any} controller $\Psi$ that satisfies the 
% assumptions of \cref{thm:main_theorem}. Thus, these results together show the 
% following implication: if there \emph{exists} a controller $\Psi$ that 
% satisfies the assumptions of \cref{thm:main_theorem}, then there is a CPWA 
% controller that satisfies the specification. And moreover, this CPWA controller 
% has at most a number of linear regions that depends only on the parameters of 
% the problem \emph{and not the particular controller $\Psi$}.
%
\noindent %As noted above,
The conclusion of \cref{thm:main_theorem} then follows from Step 1 and Step 2 
by means \cref{thm:tll_architecture} \cite[Theorem 
7]{FerlezAReNAssuredReLU2020}, since any CPWA with the same number of linear 
regions (or fewer) can be implemented exactly by a common TLL NN architecture. 

\begin{remark}\label{rem:fem_remark}
	Unlike our use of \cref{thm:tll_architecture} to directly construct a  
	single TLL, the architectures used in \cite{HeReluDeepNeural2020} consist 
	of \emph{two successive} TLL layers. The first is used to represent 
	``basis'' functions specified over overlapping regular polytopes 
	(decomposable by a common set of simplexes \cite[Figure 
	3.2]{HeReluDeepNeural2020}); and the second captures the sum of these basis 
	functions. This approach leads to larger architectures compared to our 
	direct approach: consider \cite[Eq. (5.3)]{HeReluDeepNeural2020}, which has 
	exponentially many neurons in the number of linear regions. By contrast, 
	the architecture of \cref{thm:main_theorem} has only polynomially many 
	neurons as a function of the number of regions: i.e. as a function of $N$,  
	polynomially many $\min$ units are needed \cite{FerlezAReNAssuredReLU2020}, 
	so the $\max$ network has polynomially many neurons; and each $\min$ 
	network is straightforwardly polynomial in $N$.
\end{remark}

\subsection{Proof of \cref*{thm:main_theorem}, Step 1: Approximate Controllers 
Satisfy the Specification} % (fold)
\label{sec:thm1_part1}
The goal of this section is to choose constant $\mu > 0$ such that: for any 
$\Psi$ satisfying the assumptions in \cref{thm:main_theorem}, then \emph{any 
other} controller $\Upsilon$ with $\lVert \Upsilon - \Psi \rVert_X \leq  \mu$ 
also satisfies the specification, i.e. $S_\tau(\Sigma_\Upsilon) 
\preceq_{\mathcal{AD}_0} S_\text{spec}$.

Specifically, we will prove the following lemma.
\begin{lemma}
\label{lem:final_spec}
	Let $\Sigma$ be as in \cref{prob:main_problem}. Also, let $\Psi :  X 
	\rightarrow U$ %be Lipschitz continuous with 
	be $\delta$-$\tau$ positively invariant w.r.t $X$, and have Lipschitz 
	constant at most $K_\text{cont}$.  Further suppose that $\mu > 0$ is such 
	that:
	\begin{equation}
		K_u \cdot \mu \cdot \tau \cdot
		e^{(K_x + K_u  K_\Upsilon) \tau } < \delta.
		\label{eq:approx_accuracy}
	\end{equation}

	Then for any $\Upsilon : \mathbb{R}^n \rightarrow \mathbb{R}^m$ that is 
	Lipschitz continuous with Lipschitz constant $K_\Upsilon$, we have that
	% if $S_{\tau}( \Sigma_{\Psi} ) \preceq_{\mathcal{AD}_\delta} 
	% S_\text{spec}$ and $\lVert \Upsilon - \Psi \rVert_\infty \leq \mu/3$,
	\begin{equation}
		\lVert \Upsilon - \Psi \rVert_X \leq \mu
		\implies S_{\tau}( \Sigma_{\Upsilon} )
			\preceq_{\mathcal{AD}_0}
		\mathfrak{S}_{\tau}( \Sigma_{\Psi} ).
		\label{eq:main_approximation_conclusion}
	\end{equation}
	And hence if in addition $S_{\tau}( \Sigma_{\Psi} ) 
	\preceq_{\mathcal{AD}_\delta} S_\text{spec}$, then:
	\begin{equation}
		S_{\tau}( \Sigma_{\Upsilon} ) \preceq_{\mathcal{AD}_0} S_\text{spec}.
		\label{eq:sim_conclusion_lemma}
	\end{equation}
\end{lemma}

Our specific approach to prove this %\cref{lem:final_spec}
will be as follows. 
% First, we let $x \in X$ represent an arbitrary state, and 
% we let $\mu$ be an arbitrary bound such that $\lVert \Upsilon - \Psi 
% \rVert_\infty < \mu/3 $. Then we confine our attention to a region around $x$ 
% on which the controller $\Psi$ doesn't vary much: the size of this region is 
% determined as a function of an arbitrary approximation accuracy, $\mu$, and the 
% bound on the Lipschitz constant, $K_\text{cont}$, but not $x$ itself -- the  
% subsequent claims will likewise be independent of $x$. In particular, we 
% confine a $\tau$-second flow of $\zeta_{x\Psi}$ to this region by bounding its 
% duration, i.e. by choosing $\tau$ also as a function of an  arbitrary $\mu$. 
First, we choose $\mu$ small enough such that a $\tau$-second flow of  
$\zeta_{x\Upsilon}$ doesn't deviate by more that $\delta$ from $\zeta_{x\Psi}$ 
over its duration, $\tau$. This is accomplished by means of a Gr\"onwall-type 
bound that uses $\lVert \Upsilon - \Psi \rVert_X < \mu$. That is assume $\Psi$ 
is $\delta$-$\tau$ positively invariant, and use $ \lVert \Upsilon - \Psi 
\rVert_X \leq \mu$ in the Gr\"onwall inequality:% to assert:
	\begin{equation}
		\forall x \hspace{-0.5mm} \in \hspace{-0.6mm} X . \lVert \zeta_{x\Upsilon}(\tau) \hspace{-0.5mm} - \hspace{-0.5mm} \zeta_{x\Psi}(\tau)\rVert \hspace{-0.6mm} \leq \hspace{-0.6mm} %\\
			K_u \hspace{-0.75mm}  \cdot \hspace{-0.5mm} \mu \hspace{-0.5mm} \cdot \hspace{-0.5mm} \tau \hspace{-0.6mm}  \cdot \hspace{-0.5mm}
			e^{(\hspace{-0.5mm} K_x + K_u  K_\Upsilon\hspace{-0.5mm}) \tau } \negthinspace\negthinspace < \negthinspace \delta.
		\label{eq:gronwall_conclusion}
	\end{equation} 
\noindent Then we use this conclusion to construct the appropriate simulation 
relations that show:
	\begin{equation}
		S_{\tau}( \Sigma_{\Upsilon} )
			\preceq_{\mathcal{AD}_0}
		\mathfrak{S}_{\tau}( \Sigma_{\Psi} )
			\preceq_{\mathcal{AD}_0}
		S_\text{spec}.
	\end{equation}

First, we formally obtain the desired conclusion of 
\eqref{eq:gronwall_conclusion} by way of the following proposition.

\begin{proposition}
\label{prop:delta_state_bound}
	Let $\Sigma$ and $\Psi$ be as in the statement of \cref{lem:final_spec}, 
	and let  $\Upsilon : \mathbb{R}^n \rightarrow \mathbb{R}^m$ be Lipschitz 
	continuous with Lipschitz constant $K_\Upsilon$. Also, suppose that $\mu > 
	0$ is such that:
	\begin{equation}
		K_u \cdot \mu \cdot \tau \cdot
			e^{(K_x + K_u  K_\Upsilon) \tau } < \delta .
	\end{equation}
	If $\lVert \Upsilon - \Psi \rVert_X \leq \mu$, then for any $x \in X$:
	\begin{equation}
		\lVert \zeta_{x\Upsilon}(\tau) - \zeta_{x\Psi}(\tau)\rVert \leq \delta.
	\end{equation}
\end{proposition}
% \begin{proofs}
\begin{proof}
		By definition,% and the properties of the integral, 
		we have:
	\begin{align}
		&\lVert \zeta_{x\Upsilon}(t) - \zeta_{x\Psi}(t) \rVert \notag \\
		% &=
		% \lVert
		% \int_0^t 
		% 	f(\zeta_{x\Upsilon}(\sigma),\Upsilon(\zeta_{x\Upsilon}(\sigma)))
		% 	-
		% 	f(\zeta_{x\Psi}(t), \Psi(\zeta_{x\Psi}(t)))
		% 	d\sigma
		% \rVert \notag \\
		&\negthinspace\negthinspace\leq\negthinspace\negthinspace\negthinspace
		\int_0^t 
			\lVert
			f(\zeta_{x\Upsilon}(\sigma),\Upsilon(\zeta_{x\Upsilon}(\sigma)))
			-
			f(\zeta_{x\Psi}(t), \Psi^\prime(\zeta_{x\Psi}(t)))
			\rVert
			d\sigma \notag \\
		&\negthinspace\negthinspace\leq\negthinspace\negthinspace\negthinspace
		\int_0^t \negthinspace\negthinspace\negthinspace\negthinspace\hspace{-0.1mm}
			K_{\hspace{-0.2mm}x} \hspace{-0.25mm}
			\lVert
			\hspace{-0.2mm}
				\zeta_{x\hspace{-0.2mm}\Upsilon}\hspace{-0.2mm}(\hspace{-0.25mm}\sigma\hspace{-0.35mm})
				\hspace{-0.7mm}
				-
				\hspace{-0.7mm}
				\zeta_{x\Psi}\hspace{-0.2mm}(\hspace{-0.25mm}\sigma\hspace{-0.35mm})
				\hspace{-0.25mm}
			\rVert 
		% 	\notag \\
		% &\qquad\qquad
		\hspace{-0.8mm}
		+
		\hspace{-0.8mm}
			K_{\hspace{-0.2mm}u}\hspace{-0.25mm}
			\lVert
			\hspace{-0.2mm}
				\Upsilon\hspace{-0.2mm}(\hspace{-0.2mm}\zeta_{x\hspace{-0.2mm}\Upsilon}\hspace{-0.2mm}(\hspace{-0.25mm}\sigma\hspace{-0.35mm})\hspace{-0.2mm})
				\hspace{-0.7mm}
				-
				\hspace{-0.7mm}
				\Psi\hspace{-0.2mm}(\hspace{-0.2mm}\zeta_{x\Psi}\hspace{-0.2mm}(\hspace{-0.25mm}\sigma\hspace{-0.35mm})\hspace{-0.2mm})
				\hspace{-0.25mm}
			\rVert \hspace{-0.2mm}
			d\sigma 
			\label{eq:first_gronwall_bound}
		% \notag \\
		% &\leq
		% \int_0^t 
		% 	K_x 
		% 	\lVert
		% 		\zeta_{x\Upsilon}(\sigma)
		% 		-
		% 		\zeta_{x\Psi}(\sigma)
		% 	\rVert
		% 	+
		% 	K_u
		% 	\cdot \kappa ~
		% 	d\sigma.
	\end{align}

	Now, we consider bounding the second normed quantity in 
	\eqref{eq:first_gronwall_bound}. In particular, we have that: 
	\begin{align}
			&\lVert
				\Upsilon(\zeta_{x\Upsilon}(\sigma))
				-
				\Psi(\zeta_{x\Psi}(\sigma))
			\rVert
				\notag \\
			&\;
			\leq
			\lVert
				\Upsilon(\zeta_{x\Upsilon}(\sigma))
				-
				\Upsilon(\zeta_{x\Psi}(\sigma))
			\rVert
			+
			\lVert
				\Upsilon(\zeta_{x\Psi}(\sigma))
				-
				\Psi(\zeta_{x\Psi}(\sigma))
			\rVert
			\notag \\
			&\;
			\leq 
			K_\Upsilon \cdot
			\lVert
				\zeta_{x\Upsilon}(\sigma)
				-
				\zeta_{x\Psi}(\sigma)
			\rVert
			+
			\lVert
				\Upsilon
				-
				\Psi
			\rVert_X.
			\label{eq:second_gronwal_inequality}
	\end{align}
	The first term is so bounded because $\Upsilon$ is assumed to be  globally 
	Lipschitz continuous on all of $\mathbb{R}^n$. In particular, the first 
	term of \eqref{eq:second_gronwal_inequality} can be bounded using the 
	global Lipschitz continuity of $\Upsilon$, whether $\zeta_{x\Upsilon}(t)$ 
	lies in $X$ or not. The second term is bounded as claimed because 
	$\zeta_{x\Psi}(t) \in X$ for all $t$ by the assumption of ($\delta$-$\tau$) 
	forward invariance of $\Psi$: consequently, the  $X$-restricted $\sup$ norm 
	may be employed.

	Thus, using \eqref{eq:first_gronwall_bound} and 
	\eqref{eq:second_gronwal_inequality}, we obtain the bound
	\begin{align}
		&\lVert \zeta_{x\Upsilon}(t) - \zeta_{x\Psi}(t) 
		\rVert
		% \notag \\
		% &\quad
		\leq
		\negthinspace
		\int_0^t 
			\negthinspace
			(K_x  + K_u K_\Upsilon) \negthinspace \cdot \negthinspace
			\lVert
				\zeta_{x\Upsilon}(\sigma)
				-
				\zeta_{x\Psi}(\sigma)
			\rVert \notag \\
		&\qquad\qquad\qquad\qquad \qquad\qquad
			+
			K_u \cdot
			\lVert
				\Upsilon
				-
				\Psi
			\rVert_X d\sigma.
			\label{eq:third_gronwall_inequality}
	\end{align}
	where the second term of the integrand is bounded by the constant $K_u  
	\cdot \mu$ by assumption. The claimed bound now follows from 
	\eqref{eq:third_gronwall_inequality} by the Gr\"onwall Inequality 
	\cite{KhalilNonlinearSystems2001}.
\end{proof}
% \end{proofs}

\noindent Now we can prove the main result of this section, 
\cref{lem:final_spec}; this proof follows more or less directly from  
\cref{prop:delta_state_bound}.
% \begin{lemma}
% \label{lem:final_spec}
% 	Let $\Sigma$, $\Psi$ and $\Upsilon$ be as before, and suppose that $\mu > 
% 	0$ is such that:
% 	\begin{equation}
% 		K_u \cdot \mu \cdot \frac{\mu}{6 \cdot K_\text{cont} \cdot \mathcal{K}} \cdot
% 		e^{K_x \frac{\mu}{6 \cdot K_\text{cont} \cdot \mathcal{K}}} < \delta.
% 	\end{equation}
% 	If $\lVert \Upsilon - \Psi \rVert_\infty \leq \mu/3$, then for $\tau \leq 
% 	\tfrac{\mu}{6\cdot K_\text{cont} \cdot \mathcal{K}}$ we have:
% 	\begin{equation}
% 		S_{\tau}( \Sigma_{\Upsilon} )
% 			\preceq_{\mathcal{AD}_0}
% 		\mathfrak{S}_{\tau}( \Sigma_{\Psi} ).
% 	\end{equation}
% 	And hence:
% 	\begin{equation}
% 		S_{\tau}( \Sigma_{\Upsilon} ) \preceq_{\mathcal{AD}_0} S_\text{spec}.
% 	\end{equation}
% \end{lemma}
% \begin{proofs}
\begin{proof}{(\cref{lem:final_spec})}
	 By definition, $S_{\tau}( \Sigma_{\Upsilon} )$ and $\mathfrak{S}_{\tau}( 
	\Sigma_{\Psi} )$ have the same state spaces, $X$. Thus, we propose the 
	following as an abstract disturbance simulation under $0$ disturbance (i.e. 
	a conventional simulation for metric transition systems):
	\begin{equation}
		R = \{(x,x) | x \in X\}.
	\end{equation}

	Clearly, $R$ satisfies the property that for all $(x,y) \in R$, $d(x,y) 
	\leq 0$, and for every $x\in X$, there exists an $y \in X$ such that $(x,y) 
	\in R$. Thus, it only remains to show the third property of 
	\cref{def:delta_perturbation_simulation} under $0$ disturbance.

	To wit, let $(x,x) \in R$. Then suppose that $x^\prime \triangleq 
	\zeta_{x\Upsilon}(\tau) \in X$, so that $x 
	\xition{\Upsilon}{\Upsilon\circ\zeta_{x\Upsilon}} x^\prime$ in $S_{\tau}( 
	\Sigma_{\Upsilon} )$; we will show subsequently that any such $x^\prime$ 
	must be in $X$. In this situation, it suffices to show that $x 
	\xition{{\Psi}}{\Psi\circ\zeta_{x\Psi}} x^\prime$ in $\mathfrak{S}_{\tau}( 
	\Sigma_{\Psi} )$. %, then we have shown that $R$ is an abstract disturbance 
	%simulation under $0$ disturbance.
	By the $\delta, \tau$ positive invariance of $\Psi$, it is the case that 
	$x^{\prime\prime} \triangleq \zeta_{x\Psi}(\tau) \in X \backslash 
	\text{edge}_\delta(X)$, so $x \xition{{\Psi}}{\Psi\circ\zeta_{x\Psi}} 
	x^{\prime\prime}$ in $S_\tau(\Sigma_\Psi)$. But by  
	\cref{prop:delta_state_bound}, $\lVert x^\prime - x^{\prime\prime} \rVert 
	\leq \delta$, so $x \xition{{\Psi}}{\Psi\circ\zeta_{x\Psi}} x^\prime$ in 
	$\mathfrak{S}_{\tau}( \Sigma_{\Psi} )$ by definition.

	It remains to show that any such $x^\prime \triangleq 
	\zeta_{x\Upsilon}(\tau) \in X$. This follows by contradiction from  
	\cref{prop:delta_state_bound} and the $\delta, \tau$ positive invariance of 
	$\Psi$. That is suppose $x^\prime \notin X$. By  
	\cref{prop:delta_state_bound}, $x^{\prime\prime} \triangleq 
	\zeta_{x\Psi}(\tau)$ satisfies $\lVert x^\prime - x^{\prime\prime} \rVert 
	\leq \delta$, which implies that $x^{\prime\prime} \in 
	\text{edge}_\delta(X) \cup \mathbb{R}^n \backslash X$. However, this 
	clearly contradicts the $\delta,\tau$ positive invariance of $\Psi$.
\end{proof}
% \end{proofs}
% subsubsection part_1_showing_simulation (end)

\subsection{Proof of \cref*{thm:main_theorem}, Step 2: CPWA Approximation of a 
Controller} % (fold)
\label{sec:thm1_part2}
The results in \cref{sec:thm1_part1} established a choice of $\mu$ such that 
any controller, $\Upsilon$, whether it is CPWA or not, will satisfy the 
specification if it is close to $\Psi$ in the sense that $\lVert \Upsilon - 
\Psi \rVert_X \leq \mu$; that is, provided $\Psi$ itself satisfies the 
specification by $\delta$-ADS for the chosen $\tau$. In other words, if such a 
$\Psi$ were \emph{explicitly} known, then a CPWA controller, 
$\Upsilon\subcpwa$, that meets this approximation criterion would also be a 
safe controller. In such a circumstance, that \emph{explicit} CPWA controller 
could be implemented as a ReLU by way of \cref{thm:tll_architecture}, and the  
architecture of that ReLU controller would suffice as an assured controller 
architecture.

However, we have assumed ignorance of a specific such $\Psi$, so we instead 
have to devise an architecture that can approximate any such possible $\Psi$ 
that meets the specifications. Thus, we create a \emph{parameterized} CPWA that 
can achieve the desired approximation accuracy of $\mu$ for any possible $\Psi$ 
merely by a choice of parameters. Moreover, this parameterized CPWA must have a 
bounded number of linear regions in order  to obtain a ReLU architecture by 
\cref{thm:tll_architecture}. Both of these objectives will prove feasible 
because we have assumed that any $\Psi$ has a Lipschitz constant that is 
upper-bounded by the known constant $K_\text{cont}$. Indeed, our parameterized 
CPWA will have as parameters the values of $\Psi$ on a uniformly spaced set of 
points in $X$; we will show that these function evaluations can be interpolated 
by a CPWA with a number of linear regions proportional to the number of grid 
points.
% This will yield a $\Psi$-independent parameterized CPWA, 
% since the fineness of the grid points can be determined by $K_\text{cont}$ 
% instead of any particular $\Psi$.

The mechanism for obtaining such a CPWA interpolation will ultimately be to 
``tile'' the set $X \subset \mathbb{R}^n$ by $n$-simplexes\footnote{That is 
simplexes with $n+1$ vertices.}, each of whose vertices are points from the 
interpolation grid. That is, we will choose these simplexes such that a 
full-dimensional face of any simplex will be shared  \emph{exactly} by one and 
only one other simplex. Each simplex can then be regarded as corresponding to a 
linear region, because the graph of (any) $\Psi$ on its $n+1$ vertices defines 
an affine function over that simplex. Moreover, an affine function on one 
simplex is continuous with the affine function defined on any of its neighbors, 
because of the face-sharing property described above. Thus, this approach 
constructs an interpolation CPWA where each affine function is ``activated'' 
only on its respective simplex. \emph{Crucially, this construction leads to a 
CPWA with same number of linear regions, irrespective of the particular values 
of $\Psi$}: the number of simplexes is determined entirely by the spacing of 
the uniform grid, which is in turn determined only by Lipschitz constants and 
the set $X$ itself. Hence, the grid can be specified by the problem parameters, 
and the values of $\Psi$ thereon can be considered as parameters. This is 
exactly a parameterized CPWA of the sort described above; the culmination of 
this approach is in \cref{cor:interpolation_corollary}.

\subsubsection{Preliminaries} % (fold)
\label{ssub:preliminaries}

Before we state the main result of this section, we introduce the following 
notation to refer to those hypercubes that have at least one corner in 
$X_\eta$, but no elements of $X_\eta$ in their interior.
\begin{definition}[Interpolation Hypercube]
	Let $X_\eta$ be an $\eta$-grid of $X$, and let $\mathbf{x} \in X_\eta$. 
	Then each vector $\rho \in \{-1,1\}^n$ defines an \textbf{interpolation 
	hypercube} relative to $\mathbf{x}_i$ given by:
	% \begin{equation}
	% 	C_{\mathbf{x}_i\rho} \triangleq \prod_{j=1}^n \pi_j(\rho)\cdot[\pi_j(\rho) \cdot \pi_j(\mathbf{x}_i), \pi_j(\rho) \cdot (\pi_j(\mathbf{x}_i) + \eta) ],
	% \end{equation}
	\begin{equation}
		C_{\mathbf{x}\rho} \triangleq  B( \mathbf{x} + \tfrac{\eta}{2} \cdot {\scriptstyle \sum_{i=1}^n} \pi_i(\rho) \cdot e_j)
	\end{equation}
	and whose corners are
	\begin{equation}
		\mathscr{F}_0(C_{\mathbf{x}\rho}) = \{ \mathbf{x} + \eta \cdot \sum_{i \in Z} \pi_i(\rho) \cdot e_i | Z \subseteq \{1, \dots, n\} \},
	\end{equation}
	where $e_i$ is the $i^\text{th}$ column of the $(n \times n)$ identity 
	matrix.

	Note that $\cup_{\rho \in \{-1,1\}^n} C_{\mathbf{x}\rho} = B(\mathbf{x};  
	\eta)$. \cref{fig:extra_corner_definition} is a cartoon that depicts 
	interpolation hypercubes, among other things.
\end{definition}
Interpolation hypercubes can have corners that lie outside of $X$, and we 
will subsequently need to treat these corners separately. Hence the following 
notation for an \emph{extra corner}.
\begin{definition}[Extra Corner]\label{def:extra_corner}
	Let $C_{\mathbf{x}\rho}$ be an interpolation hypercube of an $\eta$-grid, 
	$X_\eta$. Then an \textbf{extra corner} of $C_{\mathbf{x}\rho}$ is any 
	corner $x \in \mathscr{F}_0(C_{\mathbf{x}\rho})\backslash  X_\eta$. Denote 
	the set of all extra corners of all interpolation hypercubes of  $X_\eta$ as
	\begin{equation}
		\Delta_{X_\eta} \triangleq \bigcup_{\mathbf{x} \in X_\eta} \mathscr{F}_0(C_{\mathbf{x}\rho})\backslash  X_\eta.
	\end{equation}
\end{definition}

Finally, we have the following proposition, which indicates the maximum number 
of interpolation hypercubes that are required to cover $X$ (by tiling), given 
an $\eta$-grid, $X_\eta$.
\begin{proposition}\label{prop:num_interpolation_hypercubes}
	Let $X_\eta$ be an $\eta$ grid for $X$ as before. Then $X$ is covered by 
	the union of at most
	\begin{equation}\label{eq:hypercube_tiling}
		\left\lceil \frac{\text{ext}(X)}{\eta} + 2 \right\rceil^n
	\end{equation}
	interpolation hypercubes.
\end{proposition}
\begin{proof}
	A compact set $X$ can be contained in a hypercube with edge-length 
	$\text{ext}(X)$, which can in turn be partitioned by at most $\lceil  
	\text{ext}(X) / \eta \rceil^n$ elements. Thus, $|X_\eta| \leq \lceil  
	\text{ext}(X)/\eta \rceil^n$. It follows then that $X$ can be covered by 
	the inclusion of one interpolation hypercube per element of $X_\eta$, plus 
	one extra set of interpolation hypercubes at either extreme of each 
	coordinate.
\end{proof}

\begin{remark}
	Note that an $\eta$-grid of $X$ effectively specifies a tiling of no more 
	than \eqref{eq:hypercube_tiling} hypercubes covering $X$: i.e. the 
	interpolation hypercubes cover $X$, and a full dimensional face of any one 
	of them is shared by one and only one other interpolation hypercube.
\end{remark}
% \begin{proposition}
% 	Let $X_\eta$ be an $\eta$-grid of $X$, and let $C_{\mathbf{x}_i\rho}$ be an 
% 	interpolation hypercube. Then those corners of $C_{\mathbf{x}_i\rho}$ that 
% 	are contained in $X$ are elements of $X_\eta$.
% \end{proposition}
% subsubsection preliminaries (end)

\subsubsection{Parameterized CPWA for Approximation} % (fold)
\label{ssub:parameterized_cpwa_for_approximation}

Now we can state the main results of this section. The first, 
\cref{lem:parameterized_cpwa}, specifies the structure and properties of the 
parameterized CPWA that we will use. Thus, as a corollary, we have 
\cref{cor:interpolation_corollary}, which particularizes this parameterized 
CPWA to Step 2 in the proof of \cref{thm:main_theorem}.

\begin{lemma}
\label{lem:parameterized_cpwa}
	Let $\eta > 0$ be fixed, and let $X$ be compact. Also let $X_\eta$ be an 
	$\eta$-grid of $X$ that is enumerated as $X_\eta =  
	\{\mathbf{x}_i\}_{i=0}^{|X_\eta|-1}$.

	Then there is a parameterized CPWA $\Upsilon_\omega : X \rightarrow  
	\mathbb{R}$ with parameter vector $\omega \in \mathbb{R}^{|X_\eta|}$ that 
	has the following properties:

	\begin{enumerate}[{\itshape i)}]
		\item for all $i = 0, \dots |X_\eta|-1$ and $j = 0, \dots, n-1$,  
			\begin{equation}
				\pi_{i}(\omega) = \Upsilon_\omega(\mathbf{x}_{i}); 
			\end{equation}
			i.e. $\omega$ specifies the value of $\Upsilon_\omega$ on the grid 
			$X_\eta$;

		\item $\Upsilon_\omega$ has at most
			\begin{equation}
				% n! \cdot |X_\eta| \leq 
				n! \cdot \left\lceil \frac{\text{ext}(X)}{\eta} + 2 \right\rceil^n
			\end{equation}
			linear regions no matter the value of $\omega$; and

		\item let $\nu_{\mathbf{x}_i\rho} \triangleq \cup_{x \in  
			\mathscr{F}_0(C_{\mathbf{x}_i\rho})} B(x; \eta)$; then for each 
			interpolation hypercube $C_{\mathbf{x}_i\rho}$ of  $X_\eta$ and 
			each $x \in C_{\mathbf{x}_i\rho}$, 
		\begin{equation}
			\min_{\mathbf{x}^\prime \in \nu_{\mathbf{x}_i\rho} \cap X_\eta} \negthickspace\negthickspace \Upsilon_{\omega}(\mathbf{x}^\prime)
			\leq 
			\Upsilon_\omega(x)
			\leq \negthickspace
			\max_{\mathbf{x}^\prime \in \nu_{\mathbf{x}_i\rho} \cap X_\eta} \negthickspace\negthickspace \Upsilon_{\omega}(\mathbf{x}^\prime).
			\label{eq:approximation_conclusion}
		\end{equation}
	\end{enumerate}
\end{lemma}
We defer the proof of \cref{lem:parameterized_cpwa} to 
\cref{ssub:proof_of_lem:parameterized_cpwa} in order to immediately state the 
conclusion of relevance for Step 2 in the proof of \cref{thm:main_theorem}.
\begin{corollary}\label{cor:interpolation_corollary}
	Let $K_\text{cont} > 0$ be given, and let $\mu > 0$ and $\eta \leq  
	\tfrac{\mu}{{3} \cdot K_\text{cont}}$ be as above. Also let $\Psi : X 
	\rightarrow U$ be any Lipschitz continuous function with Lipschitz constant 
	at most $K_\text{cont}$.

	If we use the parameterized CPWA from \cref{lem:parameterized_cpwa} 
	component-wise, with the $j^\text{th}$ parameter vector specified by
		\begin{equation}
			\omega^{(j)} = \omega_\Psi^{(j)} \triangleq \prod_{i = 0}^{|X_\eta| - 1} \{ \pi_j(\Psi(\mathbf{x}_i)) \},
		\end{equation}
	%if for all $i = 0, \dots |X_\eta|-1$ and $j = 0, \dots, n-1$, 
	% it is the case that
	% \begin{equation}
	% 	\pi_{j+1}\circ\pi_{i}(\omega) = \pi_{(j+1)}(\Psi(\mathbf{x}_{i})),
	% \end{equation}
	then $\Upsilon_{\omega_\Psi} \triangleq \left(\begin{smallmatrix} 
	\Upsilon_{\omega_\Psi^{(1)}} & \dots & \Upsilon_{\omega_\Psi^{(m)}} 
	\end{smallmatrix}\right)$ agrees with $\Psi$ on $X_\eta$ and
	\begin{equation}\label{eq:interpolation_cor_conclusion}
		\lVert \Upsilon_{\omega_\Psi} - \Psi \rVert_X \leq \mu.
	\end{equation}
	\noindent Moreover, $\Upsilon_{\omega_\Psi}$ has at most $n! \cdot \lceil 
	\text{ext}(X)/ \eta  +2\rceil^n$ linear regions for each component-wise 
	function, independent of the particular choice of $\Psi$. Finally, 
	$\Upsilon_{\omega_\Psi}$ has a Lipschitz constant of at most  ${3}\cdot 
	K_\Psi$ on $\bigcup_{\mathbf{x}\in X_\eta, \rho \in \{-1,1\}^n} 
	C_{\mathbf{x}\rho} \supseteq X$.
\end{corollary}
\begin{proof}
	This follows from \cref{lem:parameterized_cpwa}, with the exception of 
	\eqref{eq:interpolation_cor_conclusion}, and the claim about  the Lipschitz 
	constant of $\Upsilon_{\omega_\Psi}$. 

	The claim \eqref{eq:interpolation_cor_conclusion} follows from 
	\eqref{eq:approximation_conclusion} in \cref{lem:parameterized_cpwa}, given 
	that $\eta$ was assumed chosen as $\eta \leq \tfrac{\mu}{{3} 
	K_\text{cont}}$.
	%in \cref{prop:eta_proposition}.
	In particular, we have that for any interpolation hypercube 
	$C_{\mathbf{x}_i\rho}$ and any $x \in C_{\mathbf{x}_i\rho}$, there exists a 
	$j \in \{1, \dots, m\}$ s.t. 
	\begin{equation}
		\lVert \Upsilon_{\omega_\Psi}(x) - \Psi(x) \rVert 
			=
			| \pi_j(\Psi(x)) - \Upsilon_{\omega_\Psi^{(j)}}(x) |
	\end{equation}
	From here, we consider as separate cases both possible signs of 
	$\pi_j(\Psi(x)) - \Upsilon_{\omega_\Psi^{(j)}}(x)$. However, for 
	brevity, we write out only the case  $\pi_j(\Psi(x))  - 
	\Upsilon_{\omega_\Psi^{(j)}}(x) \geq 0$, since the other case can be 
	treated by similar arguments. In this case: %, then, we obtain that:
	\begin{align}
		\lVert \Upsilon_{\omega_\Psi}(x) - \Psi(x) \rVert 
			&=
			| \pi_j(\Psi(x)) - \Upsilon_{\omega_\Psi^{(j)}}(x) | \notag \\
		~ 
			&=
			\pi_j(\Psi(x)) - \Upsilon_{\omega_\Psi^{(j)}}(x) \notag \\
	% {\scriptstyle \big(\pi_j(\Psi(x) ) \geq 0 \implies \big) \quad} 
		&\leq
	\pi_j( \Psi(x)) - \negthickspace\negthickspace\negthinspace \min_{\mathbf{x}^\prime \in 
		\nu_{\mathbf{x}_i\rho} \cap X_\eta} \negthickspace\negthickspace 
		\Upsilon_{\omega_\Psi^{(j)}}(\mathbf{x}^\prime)
	\label{eq:cor_proof_cpwa_conclusion}
	% \\
	% {\scriptstyle \big(\exists \mathbf{x} \in \mathscr{F}_0(C_{\mathbf{x}_i\rho}) \text{ s.t.}\big)\negthickspace} \quad
	% 	&\leq
	% | \pi_j(\Psi(x)) - \pi_j(\Psi(\mathbf{x})) |  \notag\\
	% 	&\leq \lVert \Psi(x) - \Psi(\mathbf{x}) \rVert \leq \mu/3
	% 	\label{eq:cor_proof_approx_conclusion}
	\end{align}
	where the \eqref{eq:cor_proof_cpwa_conclusion} follows from 
	\cref{lem:parameterized_cpwa}, equation 
	\eqref{eq:approximation_conclusion}. Then we note that there exists an  
	$\mathbf{x} \in \nu_{\mathbf{x}_i\rho} \cap X_\eta$ such that
	\begin{equation}
		\min_{\mathbf{x}^\prime \in \nu_{\mathbf{x}_i\rho} \cap X_\eta} 
		\negthickspace\negthickspace \Upsilon_{\omega_\Psi^{(j)}}(\mathbf{x}^\prime)
		=
		\Upsilon_{\omega_\Psi^{(j)}}(\mathbf{x})
		=
		\pi_j(\Psi(\mathbf{x}))
	\end{equation}
	by \emph{i)} of \cref{lem:parameterized_cpwa}. But by construction,  
	$\mathbf{x} \in \nu_{\mathbf{x}_i\rho}$, so there exists an 
	$x^{\prime\prime} \in \mathscr{F}_0(C_{\mathbf{x}_i\rho})$ such that 
	$\mathbf{x} \in  B(x^{\prime\prime};\eta)$. It follows  that $\lVert x - 
	\mathbf{x} \rVert \leq \lVert x - x^{\prime\prime} \rVert + \lVert 
	x^{\prime\prime} - \mathbf{x} \rVert \leq 2 \eta$. Thus, by the choice of 
	$\eta \leq \tfrac{\mu}{{3} K_\text{cont}}$,  $K_\text{cont}$-Lipschitz 
	continuity of $\Psi$ implies:
	%\cref{prop:eta_proposition}, we have:
	\begin{align}
		\lVert \Upsilon_{\omega_\Psi}(x) - \Psi(x) \rVert
			&\leq
		\pi_j(\Psi(x)) - \pi_j(\Psi(\mathbf{x}))  \notag\\
			&\leq \lVert \Psi(x) - \Psi(\mathbf{x}) \rVert \leq \tfrac{2}{3}\mu \leq \mu
	\end{align}
	This concludes the proof for the case $\pi_j(\Psi(x))  - 
	\pi_j(\Upsilon_{\omega_\Psi^{(j)}}(x)) \geq 0$; the other case follows 
	similarly.

	Thus, it remains to show the claim about the Lipschitz constant of  
	$\Upsilon_{\omega_\Psi}$. This ultimately follows from the observation  
	that every affine function in $\Upsilon_{\omega_\Psi}$ interpolates between 
	the values specified on the corners of some interpolation hypercube.  
	Consequently, it suffices to show that the Lipschitz constant of none of 
	these affine functions exceeds the claimed value. Thus, let $x, x^\prime 
	\in  C_{\mathbf{x}\rho}$ be arbitrary such that $x$ and $x^\prime$ belong 
	to the same simplex of the dissection $\mathcal{D}_{\mathcal{B}_n}$ (see 
	\cref{lem:cpwa_interpolate_new}). Now by linear interpolation between the 
	values of $\Upsilon_{\omega_\Psi}$ on that simplex, there exists for each 
	coordinate   $i = 1, \dots, m$  a pair of corners $x_{\min}^i, x_{\max}^i 
	\in \mathscr{F}_0(C_{\mathbf{x}\rho})$ such that
	\begin{equation}
	\pi_i( \Upsilon_{\omega_\Psi}(x_{\min}^i) ) \leq \pi_i( 
	\Upsilon_{\omega_\Psi}(x) ) \leq  \pi_i(\Upsilon_{\omega_\Psi}(x_{\max}^i)),
	\end{equation}
	and which holds for any $x$ in the same simplex. Now define:
	\begin{align}
		y_{\Delta} &\triangleq \max_i  \left[ \pi_i(\Upsilon_{\omega_\Psi}(x_{\max}^i)) -
			\pi_i(\Upsilon_{\omega_\Psi}(x_{\min}^i)) \right]; \text{ and} \\
		\label{eq:max_achieved}
		\mathsf{i}_\Delta &\triangleq \arg \max_i  \left[ \pi_i(\Upsilon_{\omega_\Psi}(x_{\max}^i)) -
			\pi_i(\Upsilon_{\omega_\Psi}(x_{\min}^i)) \right]
	\end{align}
	Assuming without loss of generality that $y_{\Delta} \neq  0$, then by 
	linearity of $\Upsilon_{\omega_\Psi}$ on the simplex containing $x$ and 
	$x^\prime$:
	\begin{equation}
		\lVert \Upsilon_{\omega_\Psi}(x) - \Upsilon_{\omega_\Psi}(x^\prime) \rVert \leq
		\frac{y_\Delta}{\eta} \lVert x - x^\prime \rVert
		\label{eq:lipschitz_bound_main_inequality}
	\end{equation}
	assuming $x_{\min}^{\mathsf{i}_\Delta}$ and $x_{\max}^{\mathsf{i}_\Delta}$ 
	are achieved on the closest corners of $C_{\mathbf{x}\rho}$, and $x - 
	x^\prime$ defines a line parallel to the  ${\mathsf{i}_\Delta}^\text{th}$ 
	coordinate.

	On the other hand, we note that
	\begin{align}
		y_\Delta &= \pi_{\mathsf{i}_\Delta} (
			\Upsilon_{\omega_\Psi} (x_{\max}^{\mathsf{i}_\Delta})
			-
			\Upsilon_{\omega_\Psi} (x_{\min}^{\mathsf{i}_\Delta})
		) \notag \\
		&= \pi_{\mathsf{i}_\Delta} (
			|
				\Upsilon_{\omega_\Psi} (x_{\max}^{\mathsf{i}_\Delta})
				-
				\Upsilon_{\omega_\Psi} (x_{\min}^{\mathsf{i}_\Delta})
			|
		) \notag \\
		&\leq
		\lVert 
			\Upsilon_{\omega_\Psi} (x_{\max}^{\mathsf{i}_\Delta})
			-
			\Upsilon_{\omega_\Psi} (x_{\min}^{\mathsf{i}_\Delta})
		\rVert.
	\end{align}
	Now note that the value of $\Upsilon_{\omega_\Psi}$ on 
	$x_{\min}^{\mathsf{i}_\Delta}$ and $x_{\max}^{\mathsf{i}_\Delta}$  is 
	specified as the value of $\Psi$ on some $\eta$-grid points no further than 
	$\eta$ away from each of those corners; recall the use of  
	$\nu_{\mathbf{x}_i \rho}$ in the statement of 
	\cref{lem:parameterized_cpwa}, \emph{iii)}. In particular, this implies that
	\begin{equation}
		\exists \mathbf{x}_{\min} \in X_\eta \;.\; 
		\lVert x_{\min}^{\mathsf{i}_\Delta} - \mathbf{x}_{\min} \rVert \leq \eta
		\wedge
		\Upsilon_{\omega_\Psi}(x_{\min}^{\mathsf{i}_\Delta}) = \Psi(\mathbf{x}_{\min})
	\end{equation}
	and likewise for $\mathbf{x}_{\max}$. Thus, we have the following: % 
	%sequence of inequalities:
	\begin{align}
		\lVert 
			\Upsilon_{\omega_\Psi} (x_{\max}^{\mathsf{i}_\Delta})
			-
			\Upsilon_{\omega_\Psi} (x_{\min}^{\mathsf{i}_\Delta})
		\rVert
		&=
		\lVert 
			\Psi(\mathbf{x}_{\max})
			-
			\Psi(\mathbf{x}_{\min})
		\rVert
		\notag \\
		&\leq
		K_\Psi \cdot
		\lVert
			\mathbf{x}_{\max}
			-
			\mathbf{x}_{\min}
		\rVert 
		\notag \\
		&\leq
		K_\Psi \cdot
		{3} \cdot \eta
		\label{eq:corner_bound}
	\end{align}
	which follow by Lipschitz continuity of $\Psi$. Substituting 
	\eqref{eq:corner_bound} into \eqref{eq:lipschitz_bound_main_inequality} 
	gives the bound on the Lipschitz constant of  $\Upsilon_{\omega_\Psi}$
\end{proof}

% subsubsection parameterized_cpwa_for_approximation (end)

\subsubsection{Proof of \cref{lem:parameterized_cpwa}} % (fold)
\label{ssub:proof_of_lem:parameterized_cpwa}
Our approach to proving \cref{lem:parameterized_cpwa} will be as follows. On  
the domain $X_\eta$, we will first \emph{define} a function $\Gamma_\omega : 
X_\eta \rightarrow \mathbb{R}$ as $\Gamma_\omega : \mathbf{x}_i \mapsto 
\pi_i(\omega)$. Then we will extend $\Gamma_\omega$ as a CPWA to the union  of 
all of the interpolation hypercubes of $X_\eta$, i.e. $\bigcup_{\mathbf{x}\in 
X_\eta, \rho \in \{-1,1\}^n} C_{\mathbf{x}\rho}$.  By this means, we %will thus 
obtain a parameterized CPWA that satisfies %the properties of 
\cref{lem:parameterized_cpwa}.

To accomplish this, we will first need to consider the problem of extending 
$\Gamma_\omega$ from the corners of a specific interpolation hypercube to the 
rest of that interpolation hypercube itself. Importantly, however, this 
interpolation must be systematized in such a way that adjacent interpolation 
hypercubes have consistent interpolations on the non-trivial faces they have in 
common. Thus, we have the following two generic lemmas, which are the last 
facilitators we will need in order to prove \cref{lem:parameterized_cpwa}. 
\cref{lem:simplex_interpolation} describes the general methodology we undertake 
for doing this: namely, dividing the hypercube into simplexes, each of which 
represents a linear region in the CPWA. \cref{lem:cpwa_interpolate_new} then 
considers a particular instance of this construction that has several unique 
properties; in particular, it is a consequence of 
\cref{lem:cpwa_interpolate_new} that will permit a globally consistent 
interpolation despite considering each interpolation hypercube individually.

\begin{lemma}\label{lem:simplex_interpolation}
	Let $C_n = [0,1]^n$, and suppose that:
	\begin{equation}
		\alpha : \mathscr{F}_0(C_n) \rightarrow \mathbb{R}
	\end{equation}
	is a function defined on the corners of $C_n$. Also suppose that  
	$\mathcal{D} = \{\mathsf{s}_i\}_{i=1}^{N_\mathcal{D}}$ is a set of 
	$n$-simplexes  such that: their vertices are contained in 
	$\mathscr{F}_0(C_n)$; they together partition $C_n$\footnote{i.e. for $i 
	\neq j$, $\mathsf{s}_i \cap \mathsf{s}_j$ is contained in lower-dimensional 
	subspace of $\mathbb{R}^n$.}; and any non-trivial face of one simplex is 
	shared exactly by one and only one other simplex. $\mathcal{D}$ is thus a 
	special case of a \textbf{dissection} of $C_n$.

	Then $\mathcal{D}$ induces a uniquely defined CPWA function  
	$\Xi_\alpha^{\mathcal{D}}: C_n \rightarrow \mathbb{R}$ such that:
	\begin{enumerate}
		\item $\forall x \in \mathscr{F}_0(C).  
			\Xi_\alpha^{\mathcal{D}}\negthinspace(x) = \alpha(x)$, i.e. 
			$\Xi_\alpha^{\mathcal{D}}$ extends $\alpha$ to $C$; and

		\item for all $x \in C$,
			\begin{equation}
				 \min_{x \in \mathscr{F}_0(C)} \alpha(x) 
				\leq \Xi_\alpha^{\mathcal{D}}\negthinspace(x) \leq \max_{x \in \mathscr{F}_0(C)} \alpha(x).
			\end{equation}
	\end{enumerate}
\end{lemma}
\begin{proof}
	This follows because each such simplex specifies $n+1$ points from 
	$\text{graph}(\alpha)$, so there is at least one hyperplane that 
	interpolates between those points in $\mathbb{R}^{n+1}$.

	In particular, let $V_{\mathsf{s}_i}$ be the vertices of the simplex 
	$\mathsf{s}_i \in \mathcal{D}$. By assumption, $V_{\mathsf{s}_i} \subset  
	\mathscr{F}_0(C_n)$, so $\alpha$ is defined on $V_{\mathsf{s}_i}$ for each 
	$\mathsf{s}_i$. Thus, let $G_i \triangleq \{ (v, \alpha(v)) \in 
	\mathbb{R}^{n+1} | v \in V_{\mathsf{s}_i} \} \subset \text{graph}(\alpha)$ 
	with $G_i$ enumerated as $G_i = \{g_{i,j}\}_{j=1}^{n+1}$. Then  
	$\mathsf{s}_i$ defines at least one hyperplane in $\mathbb{R}^{n+1}$ 
	specified by
	\begin{equation}
		h_i : x \in \mathbb{R}^{n+1} \mapsto
		\left( 
		\begin{smallmatrix}
		w_i \\ 1
		\end{smallmatrix}
		\right)^\text{T} (x - b_i) + b_i
		\label{eq:interp_hyperplane}
	\end{equation}
	where $b_i \triangleq \tfrac{1}{n+1}\sum_{g\in G_i} g$ and $w_i\in 
	\mathbb{R}^n$ solves the following (assuming without loss of  generality 
	that $\alpha$ has distinct values):
	\begin{equation}
		\left(
			\begin{matrix}
				g_{i,1}\negthinspace-\negthinspace b_i &
				\dots &
				g_{i,n+1}\negthinspace-\negthinspace b_i
			\end{matrix}
		\right)^\text{T}
		\left(\begin{smallmatrix} w_i \\ 1\end{smallmatrix}\right)
		=
		0.
	\end{equation}
	Consequently, each hyperplane, $h_i$, defines an affine function on 
	$\mathbb{R}^n$ that interpolates $\alpha$ between the relevant points in 
	its domain  (the vertices of its simplex). That is the graph of 
	\begin{equation}
		a_i : x \in \mathbb{R}^n \mapsto 
			w_i^\text{T} x - 
				\left(\begin{smallmatrix} w_i \\ 1\end{smallmatrix}\right) b_i + b_i
	\end{equation}
	contains $G_i \subset \text{graph}(\alpha)$. It also follows that
	\begin{equation}
		\forall x \in \mathsf{s}_i \;.\;
		\min_{v\in V_{\mathsf{s}_i}} \alpha(v)
		\leq
		a_i(x)
		\leq 
		\max_{v\in V_{\mathsf{s}_i}} \alpha(v),
	\end{equation}
	since $a_i$ linearly interpolates between the points $G_i$.% by construction.

	Moreover, any two such affine functions (hyperplanes) necessarily agree on 
	any non-trivial face shared by their respective simplexes. To see this, let 
	$\hat{G}_{i,i^\prime} \triangleq G_{i} \cap  G_{i^\prime} \neq \emptyset$ 
	be the vertices of such a shared face. Then for all $x \in  
	\text{hull}(\hat{G}_{i,  i^\prime})$, we have $h_i(x) = 0 =  
	h_{i^\prime}(x)$. But from the form of \eqref{eq:interp_hyperplane}, the 
	coefficient of the $n+1^\text{th}$ coordinate of $x$ is one for both $h_i$ 
	and $h_{i^\prime}$. Hence we conclude that for all $x \in \mathbb{R}^n$ and 
	$y \in \mathbb{R}$
	\begin{equation}
		\left(\begin{smallmatrix}x \\y \end{smallmatrix}\right) \in \text{hull}(\hat{G}_{i,  i^\prime}) \implies a_i(x) = y = a_{i^\prime}(x).
	\end{equation}
	Thus $a_i \hspace{-0.2mm} = \hspace{-0.2mm} a_{i^\prime}$ on  
	$\text{hull}(V_{\mathsf{s}_i})\hspace{-0.3mm} \cap \hspace{-0.3mm} 
	\text{hull}(V_{\mathsf{s}_{i^\prime}})$, the face $\mathsf{s}_i$ and 
	$\mathsf{s}_{i^\prime}$ share.

	Hence, we can define a piecewise-affine function:
	\begin{equation}
		\Xi_\alpha^\mathcal{D} : x \in C_n \mapsto \begin{cases}
			a_i(x) & \text{if } x \in \mathsf{s}_i
		\end{cases}
	\end{equation}
	with the assurance that it is a \emph{continuous} function on $C_n$ by  the 
	argument above. Of course this CPWA trivially satisfies \emph{ii)} in the 
	statement of the lemma by construction. %, as noted above.
\end{proof}
Thus, to obtain interpolate between such an $\alpha$ it suffices to construct a 
dissection of $C_n$ with the properties assumed in 
\cref{lem:simplex_interpolation}. However, to prove 
\cref{lem:parameterized_cpwa}, we use a particular dissection in order to get a 
suitable CPWA extension; specifically, we need to consider how these CPWA 
extensions interact on adjacent interpolation hypercubes, whence \emph{ii)} 
below.
\begin{lemma}
\label{lem:cpwa_interpolate_new}
	Let $C = [0,1]^n$, and suppose that:
	\begin{equation}
		\alpha : \mathscr{F}_0(C_n) \rightarrow \mathbb{R}
	\end{equation}
	is a function defined on the corners of $C_n$.

	Also, define $n\cdot(n-1)/2$ hyperplanes from the equations $b_{ij}(x) = 
	0$, $i = 1, \dots, n; j<i$, where the $b_{ij}$ are given by:
	\begin{equation}
		\left.
		\begin{matrix}
			b_{ij} : \mathbb{R}^n \rightarrow \mathbb{R} \hphantom{12345678\;} \\
			b_{ij} : x \mapsto \pi_i(x) - \pi_j(x)
		\end{matrix}
		\right\}
		\;\text{ for }\;
		\begin{matrix}
			i \in \{1, \dots, n\} \\
			j < i \hphantom{\qquad1234}
		\end{matrix}
		\label{eq:braid_definition}
	\end{equation}
	so that $\mathcal{B}_n \triangleq \{ b_{ij} \}_{i\in\{1\dots n\}; j<i}$ is 
	the so-called \textbf{braid arrangement} of hyperplanes 
	\cite{StanleyIntroductionHyperplaneArrangements}. Furthermore, let 
	\begin{equation}
		\mathcal{D}_{\mathcal{B}_n} \triangleq \{ \overline{R \cap 
		\text{int}(C_n)} | R \text{ is a region in }\mathcal{B} \}
		\label{eq:dissection_definition}
	\end{equation}
	be the dissection of $C_n$ by the regions of $\mathcal{B}$.\\

	Then $\mathcal{D}_{\mathcal{B}_n}$ is a dissection of $C_n$ by  
	$n$-simplexes, each of which has vertices only in $\mathscr{F}_0(C_n)$. 
	Thus, \cref{lem:simplex_interpolation} applies, and the resultant 
	interpolation, $\Xi_\alpha^{\mathcal{B}_n}$, has further properties:
	\begin{enumerate}[{\itshape i)}]
		\item $\Xi_\alpha^{\mathcal{B}_n}$ has at most $n!$ linear regions 
			in $C_n$;

		\item for any $n-1$-dimensional face $F_i^k \in  
			\mathscr{F}_{n-1}(C_n)$ (as in \cref{def:face_corner}) and any 
			other function $\alpha^\prime : \mathscr{F}_0(C_n) \rightarrow 
			\mathbb{R}$ we have that
			\begin{multline}
				\forall x \in (F_i^k \cap \mathscr{F}_0(C)) \; . \; \alpha(x) = \alpha^\prime(|x-e_i|) \\
				\implies
				\forall x \in F_i^k \; . \; \Xi_\alpha^{\mathcal{B}_n}\negthinspace (x) = \Xi_{\alpha^\prime}^{\mathcal{B}_n}\negthinspace (|x - e_i|)
			\end{multline}
		% \item for any other function $\alpha^\prime : \mathscr{F}_0(C) 
		% 	\rightarrow \mathbb{R}$ with the property that

		% 	\begin{itemize}
		% 		\item $\alpha(x) = \alpha^\prime(|x-e_i|)$ for all $x$ in  
		% 			$F_i^k \cap \mathscr{F}_0(C)$ (where $F_i^k \in  
		% 			\mathscr{F}_{n-1}(C)$; see \cref{def:face_corner})
		% 	\end{itemize}
		% 	then 
		% 	\begin{equation}
		% 		\forall x \in F_i^k \; . \; \Xi_\alpha(x) = \Xi_{\alpha^\prime}(|x - e_i|);
		% 	\end{equation}
			i.e. if $\alpha$ and $\alpha^\prime$ agree on the corners of  
			\textbf{opposing} $n-1$-dimensional faces $F_i^k$ and 
			$F_i^{|k-1|}$, then the interpolations $\Xi_\alpha^{\mathcal{B}_n}$ 
			and $\Xi_{\alpha^\prime}^{\mathcal{B}_n}$ agree on the entirety of 
			those faces (that is agree under translation by $e_i$).
	\end{enumerate}

	% Furthermore, let $\Gamma_C^\prime : \mathscr{F}_0(C) \rightarrow 
	% \mathbb{R}$ be another function such that for some $F \in  
	% \mathscr{F}_{n-1}(C)$, $\Gamma_C^\prime(x + e_i) + \Gamma_C(x)$ for all  $x 
	% \in F$.

	% Furthermore, let $C^{i} \triangleq C + e_i$ where $e_i$ the $i^\text{th}$ 
	% column of the $( n \times n)$ identity matrix (i.e.  $C^{i}$ is a 
	% translation of $C$ along the  $i^\text{th}$ coordinate axis). Also, let 
	% $\Gamma_{C^{i}} : \mathscr{F}_0(C^i) \rightarrow \mathbb{R}$ s.t. for all  
	% $x \in \mathscr{F}_0(C) \cap \mathscr{F}_0(C^i)$, $\Gamma_C(x) =  
	% \Gamma_{C^i}(x)$. Then
\end{lemma}
\noindent We defer the proof of \cref{lem:cpwa_interpolate_new} until 
\cref{sub:proof_of_cpwa_interpolate_new}. However, for context, the dissection 
obtained from the braid arrangement is shown below in 
\cref{fig:braid_dissection_3d} for $n=3$.
\begin{figure}[h]
	%\hspace{5.4mm}
	\begin{tabular}{cc}
		\begin{minipage}[t]{0.15\textwidth}
			\vspace{-100pt}
			\includegraphics[width=\textwidth, trim={0 0 0 0}, 
			clip]{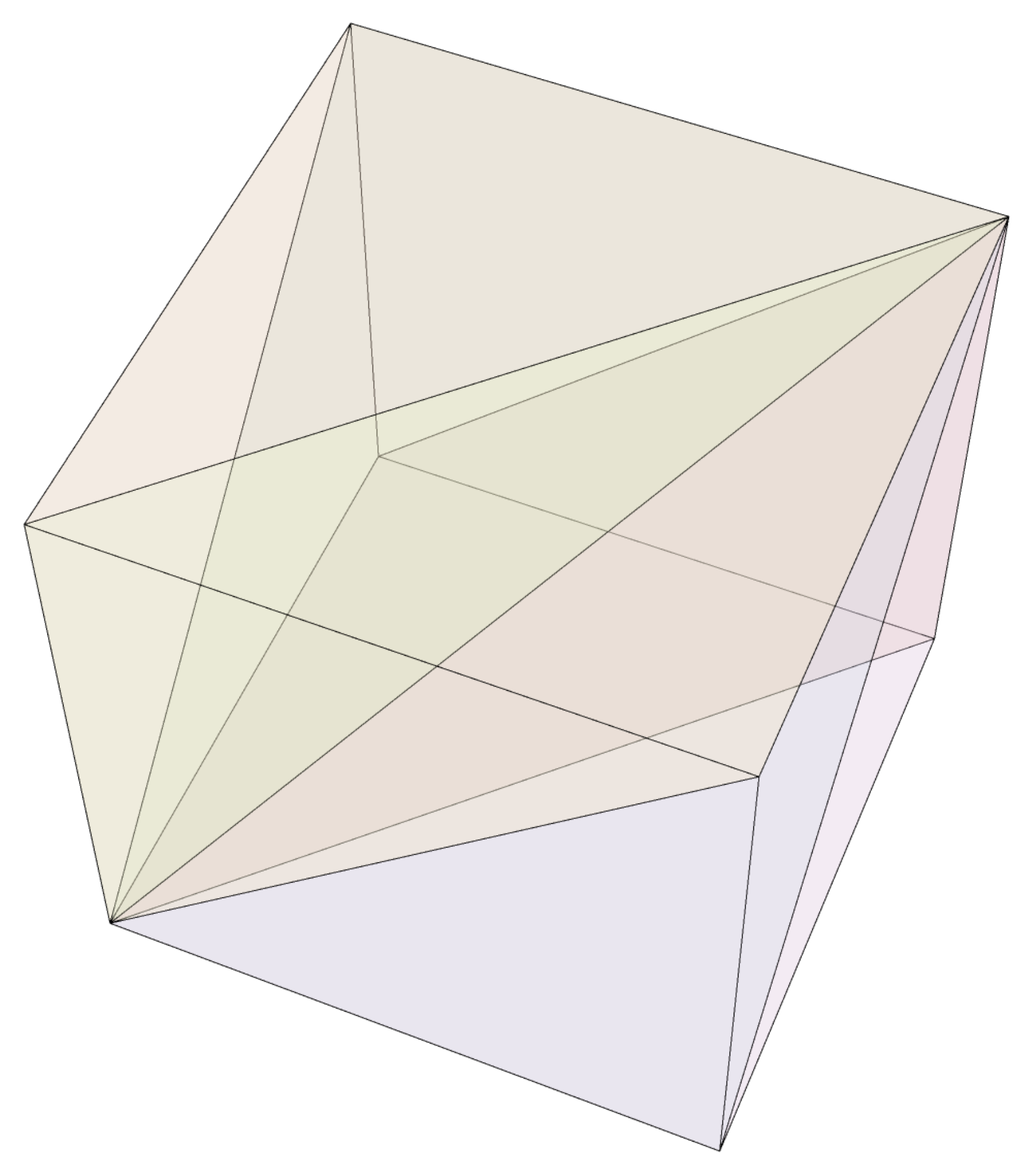} 
		\end{minipage}
		& %
		\includegraphics[width=0.3\textwidth, trim={0 100pt 0 100pt}, clip]{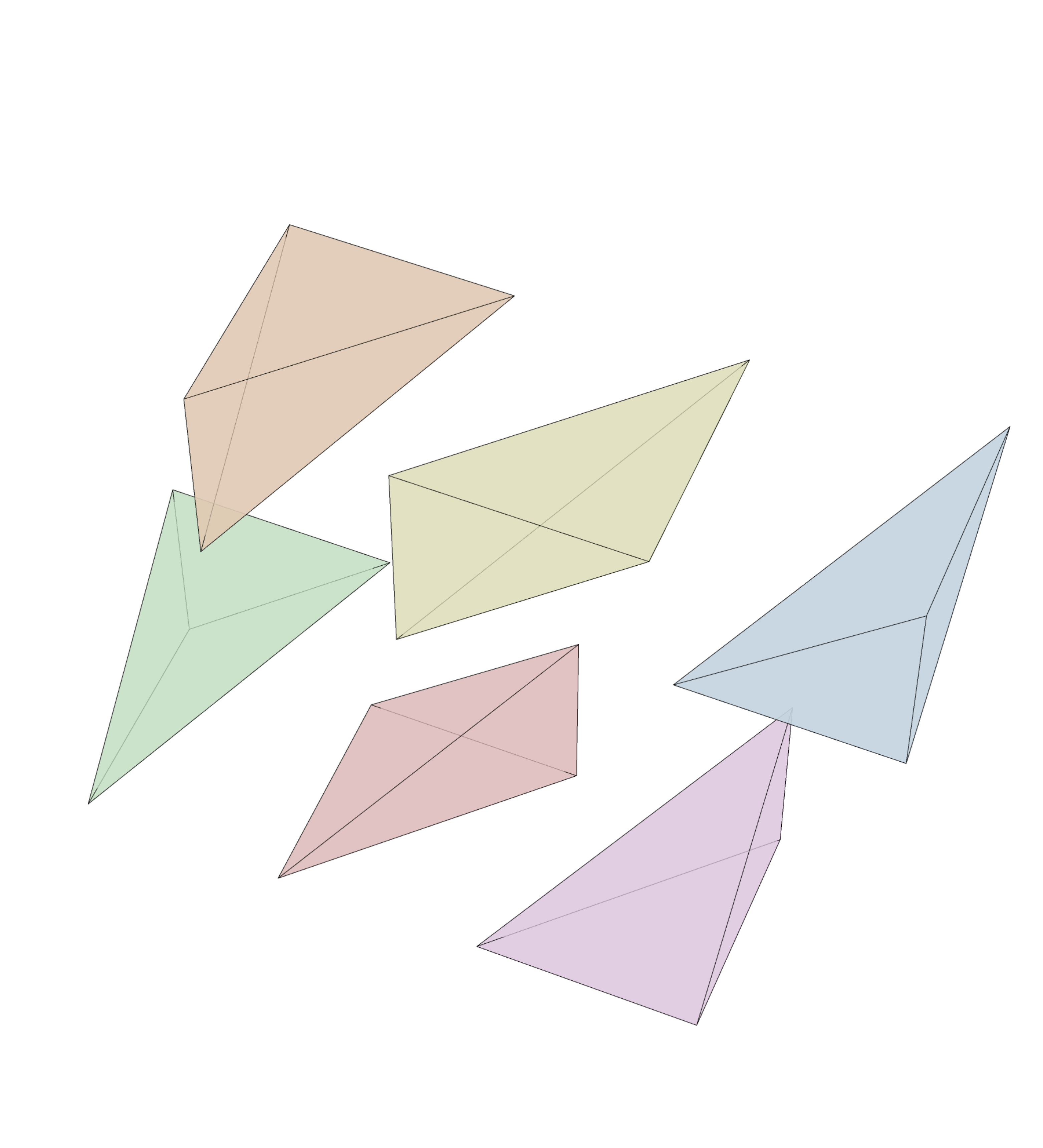}
	\end{tabular} \vspace{-1.5mm}
	\caption{Dissection of $C_3$ by the braid hyperplane arrangement, $\mathcal{B}_3$, to yield $\mathcal{D}_{\mathcal{B}_3}$; an exploded view depicts the individual simplexes more clearly. See \cref{lem:cpwa_interpolate_new}.} 
	\label{fig:braid_dissection_3d} %\vspace{-2.5mm}
\end{figure}

\noindent \cref{fig:braid_dissection_3d} also illustrates a property of 
$\mathcal{D}_{\mathcal{B}_n}$ that is crucial to establish \emph{ii)} in 
\cref{lem:cpwa_interpolate_new}: note that each face is split into two 
triangles, and each of those triangles is a single-coordinate translation of a 
triangle on the opposite face. This fact is considered in the proof of 
\cref{lem:cpwa_interpolate_new} in \cref{sub:proof_of_cpwa_interpolate_new}.

\begin{remark}
	Property ii) of \cref{lem:cpwa_interpolate_new} can also be interpreted as 
	a consequence of the following observation: the braid dissection of a 
	hypercube results in ``symmetric'' dissections on opposite  
	$n-1$-dimensional faces (c.f. \cref{fig:braid_dissection_3d}). Thus, 
	dissecting each individual interpolation hypercube this way yields a 
	collection of simplexes that ``tile'' their union; i.e. a full-dimensional 
	face of any simplex is shared  \emph{exactly} by one and only one other 
	simplex -- hence the desired continuity.
\end{remark}

\begin{remark}
	The braid arrangement dissection is a particular dissection by Sch\"afli 
	orthoschema. \cite{CoxeterRegularPolytopes1973}
\end{remark}

Now %with \cref{lem:simplex_interpolation} and \cref{lem:cpwa_interpolate_new} 
%in hand, 
we can finally state the proof of \cref{lem:parameterized_cpwa}.
\begin{proof}
	\emph{(\cref{lem:parameterized_cpwa})} The proof will proceed as suggested 
	above. First, we will define the function $\Upsilon_\omega$ on the elements 
	of $X_\eta$ using the parameter vector $\omega$. Then we will extend 
	$\Upsilon_\omega$ to the corners of all the interpolation hypercubes 
	associated with $X_\eta$; i.e. we will define $\Upsilon_\omega$ on the 
	extra corners of the aforementioned interpolation hypercubes, again in 
	terms of $\omega$ only (see \cref{def:extra_corner}). Then we will employ 
	\cref{lem:simplex_interpolation} and \cref{lem:cpwa_interpolate_new} to 
	define a CPWA within each interpolation hypercube that interpolates between 
	the values on its corners -- those values being derived exclusively from 
	$\omega$ by the construction above.  Finally, 
	\cref{lem:cpwa_interpolate_new} can be used to show that this interpolation 
	scheme is consistent from one interpolation hypercube to the next, thus 
	making it a viable CPWA defined over the entirety of $X$. Moreover, it will 
	also lead to the claimed bound on the number of linear regions needed.

	To this end, first define $\Upsilon_\omega$ on $X_\eta$ as follows.  For 
	each $\mathbf{x}_i \in X_\eta =  \{\mathbf{x}_i\}_{i=0}^{\scriptscriptstyle 
	|X_\eta|-1}$, define: $\Upsilon_\omega(\mathbf{x}_i) \triangleq 
	\pi_i(\omega)$. Now, we have to define $\Upsilon_\omega$ to the extra 
	corners associated with the interpolation hypercubes of $X_\eta$, i.e.  
	$\Delta_{X_\eta}$ (see \cref{def:extra_corner}). In particular, we extend  
	$\Upsilon_\omega$ to $\Delta_{X_\eta}$ using the following definition:
	\begin{equation}\label{eq:upsilon_extra_corners}
		\forall x \in \Delta_{X_\eta} \; . \; \Upsilon_\omega(x) \triangleq 
		\min_{\mathbf{x} \in B(x; \eta)\cap X_\eta} \Upsilon_\omega(\mathbf{x}).
	\end{equation}
	Note that \eqref{eq:upsilon_extra_corners} is well defined, since $x\in  
	\Delta_{X_\eta}$ is the corner of an interpolation hypercube -- which has 
	edge length $\eta$ -- and each interpolation hypercube has at least one 
	corner that belongs to $X_\eta$; see \cref{fig:extra_corner_definition}. 
	Moreover, we have just defined $\Upsilon_\omega$ for any such $\mathbf{x} 
	\in X_\eta$ in terms of $\omega$ only, so this scheme for defining 
	$\Upsilon_\omega$ on $\Delta_{X_\eta}$ likewise depends only on $\omega$.
	\begin{figure}[h]
		\centering
		% \vspace{1.4mm}
		\includegraphics[width=0.35\textwidth, trim={0 0 0 0},clip]{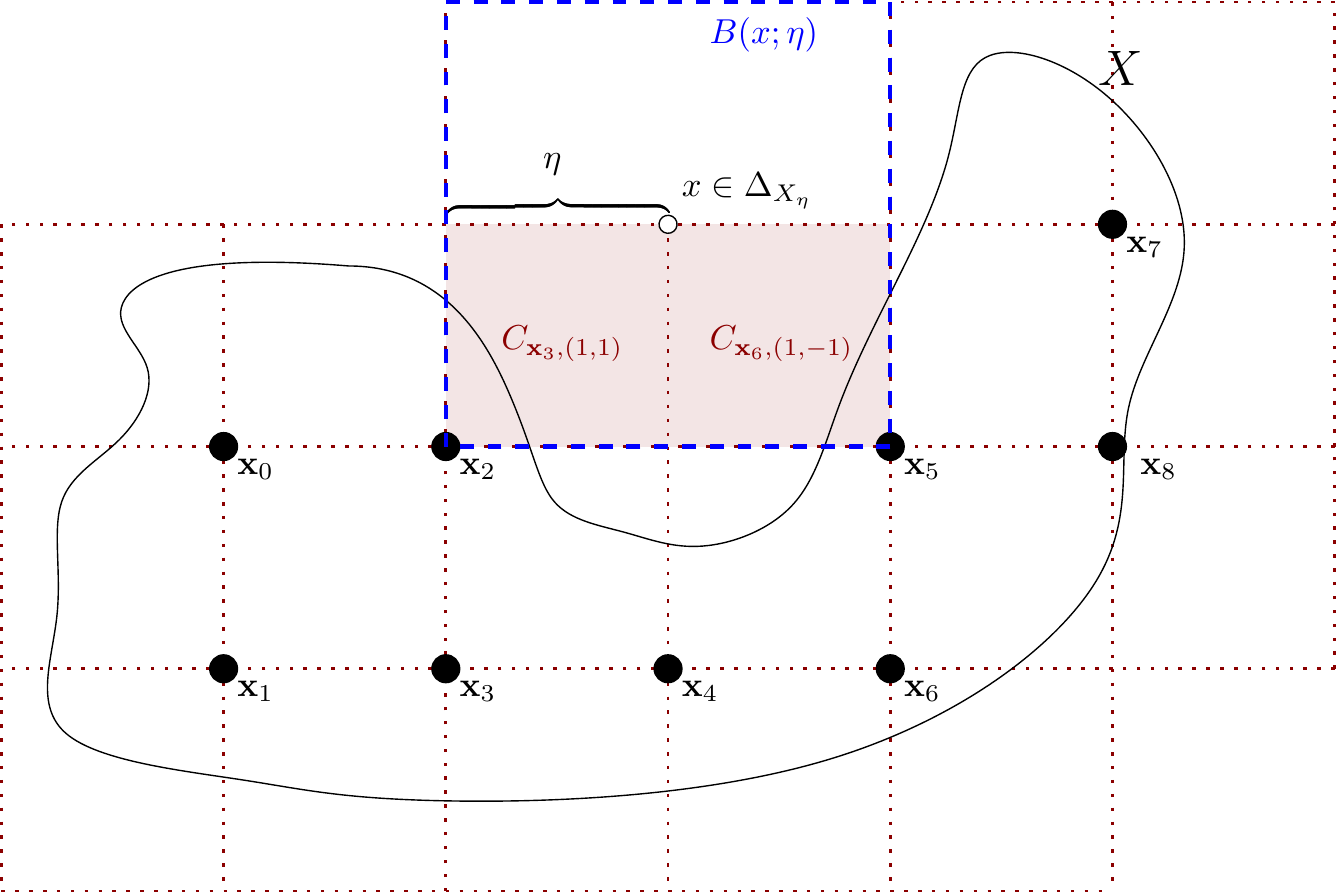} %
		\caption{Interpolation Hypercubes. Solid dots denote members of $X_\eta$; red dotted lines denote interpolation hypercubes; and $x\in\Delta_{X_\eta}$ is an extra corner (see \cref{def:extra_corner}).} 
		\label{fig:extra_corner_definition} %\vspace{-2.5mm}
	\end{figure}

	The previous construction ensures that $\Upsilon_\omega$ is defined on the 
	corners of each interpolation hypercube associated with $X_\eta$, and thus 
	\cref{lem:simplex_interpolation} and \cref{lem:cpwa_interpolate_new} may be 
	applied to obtain a number of CPWAs, each of whose domains is a single 
	interpolation hypercube. Crucially, however, if these CPWAs are continuous 
	from one interpolation hypercube to any adjacent one, then they may 
	together be considered to define a CPWA whose domain is the union of all 
	such interpolation hypercubes (simply switching between them as  
	appropriate). The result is a CPWA that contains $X$ in its domain.  Then, 
	since each individual CPWA requires at most $n!$ linear regions (by 
	\cref{lem:cpwa_interpolate_new}), this construction leads to the bound in 
	claim \emph{ii)} of the Lemma simply by using % the bound from 
	\cref{prop:num_interpolation_hypercubes}.

	Thus, it remains to show that the CPWAs defined on each interpolation 
	hypercube are consistent (continuous) from one to the next. However, this 
	follows more or less directly from property \emph{ii)} of 
	\cref{lem:cpwa_interpolate_new}. To see this, let $C_{\mathbf{x}\rho}$ be 
	an interpolation hypercube, and let $C_{\mathbf{x}^\prime\rho^\prime}$ be 
	any other interpolation hypercube that shares an $n-1$-dimensional face, 
	$F_i^k$, with $C_{\mathbf{x}\rho}$. In particular, both 
	$\Upsilon_\omega|_{C_{\mathbf{x}\rho}}$ and 
	$\Upsilon_\omega|_{C_{\mathbf{x}^\prime\rho^\prime}}$ can be regarded as 
	inducing separate functions on the unit hypercube $C_n$ (by translation), 
	and because $\Upsilon_\omega$ is defined on the corners of all such 
	interpolation hypercubes, these functions satisfy the condition of property 
	\emph{ii)} of \cref{lem:cpwa_interpolate_new}. Thus, their interpolations 
	are identical on opposite faces, which yields the claimed continuity under 
	translation back to domains $C_{\mathbf{x}\rho}$ and 
	$C_{\mathbf{x}^\prime\rho^\prime}$.% This completes the proof.
	% \begin{itemize}
	% 	\item Extend $\Gamma_\omega$ to the extra corners of the 
	% 		interpolation hypercubes. (NB: we have to worry about those extra  
	% 		corners that belong to multiple interpolation hypercubes; for one 
	% 		of those, define $\Gamma_\omega$ as a point in $X_\eta$ shared by  
	% 		all the hypercubes it touches, or if there are none, just pick in  
	% 		a systematic way one point from one of the hypercubes that it 
	% 		\emph{does} touch.)
	% 	\item Apply \cref{lem:cpwa_interpolate_new} to each of those 
	% 		hypercubes, and show that those interpolations are consistent
	% \end{itemize}
\end{proof}

% \begin{definition}[Neighboring Grid Point]
% 	Let $X_\eta$ be an $\eta$-grid of $X$, and let $\mathbf{x}_i \in X_\eta$. 
% 	Then the set of neighboring grid points of $\mathbf{x}_i$ is the set  
% 	\begin{multline}
% 		\mathcal{N}(\mathbf{x}_i) \triangleq \\
% 		\left\{
% 			\mathbf{x} \in X_\eta
% 			\; | \;
% 			\exists E \in \{-1, 0, 1\}^n . \big( \mathbf{x} = \mathbf{x}_i + \eta \cdot E \big)
% 		\right\}.
% 	\end{multline}
% \end{definition}

% subsubsection proof_of_lem:parameterized_cpwa (end)

\subsection{Proof of \cref*{thm:main_theorem}}
\label{sec:proof_of_main_theorem}
Using the consequences of Step 1 (see \cref{sec:thm1_part1}) and Step 2 (see 
\cref{sec:thm1_part2}), we can now prove \cref{thm:main_theorem}.
\begin{proof}
	By \cref{cor:interpolation_corollary}, there is an %a parameterized CPWA,  
	$\Upsilon_{\omega_\Psi}$ that:
	\begin{itemize}
		\item approximates any controller $\Psi$ to the desired accuracy  
			via $\lVert \cdot \rVert_X$ merely by changing its parameter  
			values;

		\item requires no more than the claimed number of linear regions on 
			$X$, given this accuracy, independent of $\Psi$; and

		\item has a Lipschitz constant at most ${3} \cdot K_\text{cont}$ on  
			$X$, also independent of $\Psi$.
	\end{itemize}
	By \cref{cor:parallel_tll} (a consequence of \cite[Theorem 
	2]{FerlezAReNAssuredReLU2020}) we can find a TLL NN architecture to 
	implement the CPWA $\Upsilon_{\omega_\Psi}$. Moreover, if such a TLL is 
	implemented using only linear functions whose Lipschitz constants are 
	bounded by ${3}\cdot K_\text{cont}$, as is the case to  implement 
	$\Upsilon_{\omega_\Psi}$ exactly on $X$, then that TLL will likewise have a 
	\emph{global} Lipschitz constant of at most ${3} \cdot  K_\text{cont}$. 
	Thus, \cref{lem:final_spec} applies, and that completes the proof.
\end{proof}

% section a_relu_architecture_for_nonlinear_systems (end)

 %

% !TEX root = ./main.tex

\section{Extension: ReLU Architectures for System Identification} % (fold)
\label{sec:relu_architectures_for_system_identification}

The methodology of the previous section can be summarized in the following way: 
determine the amount of (instantaneous) controller error that can be tolerated 
relative to a robust controller (Step 1), and design an architecture that can 
achieve that error for an \emph{arbitrary} robust controller (Step 2). In  
particular, though, the methodology of Step 1 -- chiefly the Gr\"onwall 
inequality -- only recognizes instantaneous controller error as an error in the 
(controlled) vector field. As a result, essentially the same calculations can 
be recycled when the source of error comes not from a different controller, but 
an \emph{alternate, erroneous vector field} under the same controller. 

Indeed, this is almost a complete program to repurpose the results of 
\cref{sec:a_relu_architecture_for_nonlinear_systems} to obtain assured 
architectures for system identification. In this context, we treat the 
instantaneous error as coming from the fact that the \emph{identified} vector 
field -- a NN in this case -- can't exactly replicate the vector field on which 
it is trained. Nevertheless, from Step 2, we can find a sufficiently large NN 
system architecture that can be trained to a proscribed uniform approximation 
accuracy (relative to the unknown system). This architecture   leads to 
closed-loop assurances under system identification, since a Gr\"onwall 
calculation -- like \cref{prop:delta_state_bound} of Step 1 -- connects 
training error to a controller robustness margin (and conversely).

In fact, this proposed architecture will lead to something like the following 
assurance: assuming the architecture has been trained adequately, then an 
appropriately robust controller designed for the \emph{identified system} will 
also control the \emph{original, unknown system}. In this way, the proposed 
system identification architecture comes with both a (uniform) training 
accuracy requirement and a robustness requirement for controllers designed to 
meet a specification. If these requirements can be met by the identified  
architecture and a controller designed for said \emph{identified} architecture, 
then the very same controller is assured to control the original, unknown 
system.

This section begins with 
\cref{sub:problem_formulation_nn_architectures_for_systemid}, in which we 
formulate an assured NN architecture design problem for system identification. 
Then, in \cref{sub:relu_architectures_for_nonlinear_system_identification}, we 
present \cref{thm:main_theorem_systemid}, which provides just such an assured 
NN architecture.% for system identification.

\subsection{Problem Formulation: NN Architectures for System Identification} % (fold)
\label{sub:problem_formulation_nn_architectures_for_systemid}
As suggested, we will formulate the problem of designing assured architectures 
for system identification in roughly the \emph{reverse} way as the controller 
architecture problem, \cref{prob:main_problem}. That is we start from a given 
accuracy to which the identified dynamics can be learned relative to the 
original, unknown dynamics; this accuracy will be denoted by the suggestive 
symbol $\mu > 0$. Then the problem will be to find an identified architecture 
\emph{and} robustness requirements $\delta$ and $\tau$ such that: if the 
identified architecture is learned to a uniform error of $\mu$, then any 
$\delta$-robust controller designed on the \emph{identified} system will 
satisfy the same specification on the original system (non-robustly).  Hence, 
we have the following problem of designing assured ReLU architectures for 
system identification.

\begin{problem}[ReLU Architectures for System Identification]
\label{prob:system_id_problem}
	Let $\mu \negthinspace > \negthinspace 0$ and $K_\text{cont} \negthinspace 
	> \negthinspace 0$ be given. Let $\Sigma$ be a feedback controllable 
	Lipschitz control system, and let $S_\text{spec} \negthinspace = 
	\negthinspace (X_\text{spec},U_\text{spec}, 
	\ltsxition{S_\text{spec}\negthickspace\negthickspace}{~})$ be a transition 
	system specification on $\Sigma$.

	Then the problem is to find $\tau = \tau(K_x,K_u,K_\text{cont},\mu)$ and 
	$\delta =  \delta(K_x,K_u,K_\text{cont},\mu)$ together with a ReLU 
	architecture, $\text{Arch}(\Theta)$, which corresponds to an 
	\textbf{identified} system
	\begin{equation}
	\label{eq:identified_system}
		\Sigma^{\Theta} = ( X, U, \mathcal{U}, \nn_{\Theta}),
	\end{equation}
	where $X$, $U$ and $\mathcal{U}$ are as in $\Sigma$, and $\nn_\theta : X  
	\times U \rightarrow X$. The identified system, $\Sigma^{\Theta}$, must 
	further have the properties that:
	\begin{enumerate}[{\itshape i)}]
		\item there exists values of $\Theta$ s.t.
			\begin{equation}\label{eq:dynamics_approx}
				\lVert f - \nn_{\Theta} \rVert \leq \mu
			\end{equation}
			i.e. the architecture is able to learn the vector field $f$ to the 
			proscribed accuracy; and

		\item for $\Theta$ as in \eqref{eq:dynamics_approx}, if $\Psi: X 
			\rightarrow U$ is a $\delta$-$\tau$ invariant Lipschitz-continuous 
			controller with $K_\Psi \le K_\text{cont}$ s.t.
			\begin{equation}
				S_{\tau}(\Sigma^\Theta_\Psi) \preceq_{\mathcal{AD}_\delta} S_\text{spec},
			\end{equation}
			then
			\begin{equation}
				S_{\tau}(\Sigma_{\Psi})
					\preceq_{\mathcal{AD}_0}
				S_\text{spec}
			\end{equation}
			i.e. if $\Psi$ controls the \textbf{identified} system to meet the 
			specification robustly, then $\Psi$ also controls the original 
			system, $S(\Sigma_\Psi)$, to the same specification (albeit 
			non-robustly).
	\end{enumerate}
\end{problem}

% subsection problem_formulation_relu_architectures_for_systemid (end)

\subsection{ReLU Architectures for Nonlinear-System Identification} % (fold)
\label{sub:relu_architectures_for_nonlinear_system_identification}

Now we have the following theorem, which solves \cref{prob:system_id_problem} 
%by describing ReLU architecture
with the claimed guarantees for system 
identification.
\begin{theorem}
\label{thm:main_theorem_systemid}
Let $\mu > 0$ and $K_\text{cont} > 0$ be given, and let $\Sigma$ and  
$S_\text{spec}$ be as in \cref{prob:system_id_problem}. Furthermore, let  
$\delta,\tau > 0$ be s.t.
\begin{equation}
\label{eq:sys_id_delta_tau}
	K_u \cdot \mu \cdot \tau \cdot
		e^{(K_x + K_u  K_\text{cont}) \tau } < \delta.
\end{equation}

If $\eta \leq \tfrac{\mu}{{3} \cdot K_\text{cont}}$ is such that there exists 
an $\eta$-grid of $X \times U$, then an $n$-fold TLL NN architecture 
$\text{Arch}(\tllThetaPar_{^{n+m}_n\negthinspace N})$ of size: 
	\begin{equation}\label{eq:sys_id_size}
		N \geq 
		% \left(
		% 	n! \cdot \sum_{k=1}^{n} \frac{2^{2k-1}}{(n-k)!}
		% \right)
		(n + m)!
		\cdot
		\left\lceil\frac{\text{ext}(X \times U)}{\eta} + 2\right\rceil^{n+m}
	\end{equation}
defines an \textbf{identified system} as in \eqref{eq:identified_system} of 
\cref{prob:system_id_problem}, which has the following properties:

\begin{enumerate}[{\itshape i)}]
	\item there exist values for $\tllThetaPar_{^{n+m}_n\negthinspace N}$ 
		such that
		\begin{equation}
		\lVert f - \nn  {\scriptstyle 
		\tllThetaPar_{^{n+m}_n\negthinspace N}} \rVert_{X\times U} < \mu; 
		\text{ and}
		\end{equation}

	\item for this $\tllThetaPar_{^{n+m}_n\negthinspace N}$  and $\Psi : X 
		\rightarrow U$ $\delta$-$\tau$ positive invariant w.r.t to  $X$ and 
		Lipschitz continuous with $K_\Psi \leq K_\text{cont}$, 
			\begin{equation}
				S_{\tau}(\Sigma_{\Psi}\negthinspace\negthinspace\negthinspace{{\scriptstyle \tllThetaPar_{\scriptscriptstyle ^{n+m}_n\negthinspace N}} \atop \vphantom{{1 \atop 2}}} \negthinspace)
					\preceq_{\mathcal{AD}_\delta} \negthinspace
				S_\text{spec}
				\negthinspace\negthinspace \implies \negthinspace\negthinspace
				S_{\tau}(\Sigma_{\Psi} \negthinspace)
					\preceq_{\mathcal{AD}_0} \negthinspace
				S_\text{spec}.
			\end{equation}
\end{enumerate}
\vspace{-5pt}
 For convenience, denote a trajectory of
	$\Sigma_{\Psi}\negthinspace\negthinspace\negthinspace{{\scriptstyle 
	\tllThetaPar_{\scriptscriptstyle ^{n+m}_n\negthinspace N}} \atop 
	\vphantom{{1 \atop 2}}}$
by $\zeta_{x\Psi}^{\nn}$, and\\[-5pt]
the corresponding vector field by $f^{\nn}$.
\end{theorem}

\begin{proof}
	From the results in \cref{sec:thm1_part2}, particularly 
	\cref{cor:interpolation_corollary}, we know that a parallel TLL 
	architecture of size determined by \eqref{eq:sys_id_size} can indeed 
	approximate the unknown dynamics as required by conclusion \emph{i)}  
	above. Thus, it remains to show that if this NN is trained to that accuracy 
	with respect to (unknown) $f$, then conclusion \emph{ii)} holds for a 
	controller, $\Psi$.

	To this end, let $\Psi$ be such a controller that is $\delta$-$\tau$ 
	positive invariant, and where $\delta$ and $\tau$ are selected according to 
	the assumptions of the Theorem. Moreover, suppose that $f^{\nn}$ has been 
	trained so that $\lVert f - f^{\nn} \rVert \leq \mu$ for $\mu$ as specified.

	Thus, we can proceed as in the proof of \cref{lem:final_spec}. In 
	particular, we seek to bound the difference
	\begin{equation}
	\label{eq:sys_id_sim_requirement}
		\lVert \zeta_{x\Psi}(\tau) - \zeta_{x\Psi}^{\nn}(\tau) \rVert < \delta
	\end{equation}
	so that we can construct the appropriate simulation relations. To prove  
	this bound, observe that:
	\begin{align}
		&\lVert \zeta_{x\Psi}(t) - \zeta_{x\Psi}^{\nn}(t) \rVert \notag \\
		% &=
		% \lVert
		% \int_0^t 
		% 	f(\zeta_{x\Psi}(\sigma),\Psi(\zeta_{x\Psi}(\sigma)))
		% 	-
		% 	f^{\nn}(\zeta_{x\Psi}^\nn(t), \Psi(\zeta_{x\Psi}^\nn(t)))
		% 	d\sigma
		% \rVert \notag \\
		% &\leq
		% \int_0^t 
		% 	\lVert
		% 	f(\zeta_{x\Psi}(\sigma),\Psi(\zeta_{x\Psi}(\sigma)))
		% 	-
		% 	f^{\nn}(\zeta_{x\Psi}^{\nn}(\sigma), \Psi(\zeta_{x\Psi}^{\nn}(\sigma)))
		% 	\rVert
		% 	d\sigma \notag \\
		&\negthinspace\negthinspace\leq\negthinspace\negthinspace\negthinspace
		\int_0^t 
			\lVert
				f(\zeta_{x\Psi}(\sigma),\Psi(\zeta_{x\Psi}(\sigma)))
				-
				f(\zeta_{x\Psi}^{\nn}(\sigma),\Psi(\zeta_{x\Psi}^{\nn}(\sigma)))
			\rVert \notag \\
		& \negthinspace+
			\lVert
				f(\zeta_{x\Psi}^{\nn}(\sigma),\Psi(\zeta_{x\Psi}^{\nn}(\sigma)))
				-
				f^{\nn}(\zeta_{x\Psi}^{\nn}(\sigma),\Psi(\zeta_{x\Psi}^{\nn}(\sigma)))
			\rVert
			d\sigma \notag \\
		&\negthinspace\negthinspace\leq\negthinspace\negthinspace\negthinspace
		\negthinspace
		\int_0^t \hspace{-1.6mm}
			(\hspace{-0.3mm} K_x \hspace{-0.5mm} + \hspace{-0.5mm} K_u K_\text{cont} \hspace{-0.3mm}) \hspace{-1.0mm} \cdot \hspace{-1.0mm}
			\lVert
				\zeta_{x\Psi}\hspace{-0.3mm}(\hspace{-0.3mm}\sigma\hspace{-0.3mm})
				\hspace{-0.4mm}
				-
				\hspace{-0.35mm}
				\zeta_{x\Psi}^{\nn}\hspace{-0.4mm}(\hspace{-0.3mm}\sigma\hspace{-0.3mm})
				\hspace{-0.3mm}
			\rVert \hspace{-1.0mm} + \hspace{-1.0mm }
		% 	\notag \\
		% &\qquad\qquad\qquad\qquad\qquad\qquad+
			\lVert
				\hspace{-0.2mm}
				f
				\hspace{-0.65mm}
				-
				\hspace{-0.65mm}
				f^{\nn}
				\hspace{-0.35mm}
			\rVert_{\hspace{-0.4mm}\scriptscriptstyle X\hspace{-0.35mm}\times\hspace{-0.3mm} U} \hspace{-0.2mm}
			d\sigma
			\label{eq:sys_id_gronwall_bound}
		% \notag \\
		% &\leq
		% \int_0^t 
		% 	K_x 
		% 	\lVert
		% 		\zeta_{x\Upsilon}(\sigma)
		% 		-
		% 		\zeta_{x\Psi}(\sigma)
		% 	\rVert
		% 	+
		% 	K_u
		% 	\cdot \kappa ~
		% 	d\sigma.
	\end{align}
	where the last inequality follows from the fact that $\Psi$ was assumed to 
	be \emph{designed} such that it is $\delta$-$\tau$ positive invariant with 
	respect to $X$ (and by construction $\Psi$ has codomain $U$). Of course the 
	Gr\"onwall inequality applies to \eqref{eq:sys_id_gronwall_bound} when we  
	use the assumed training error bound $\lVert f - f^{\nn} \rVert \leq  \mu$. 
	Since $\delta$ and $\tau$ were chosen according to 
	\eqref{eq:sys_id_delta_tau} -- which has exactly the form of the  
	Gr\"onwall bound obtained in \eqref{eq:sys_id_gronwall_bound} -- we 
	conclude that \eqref{eq:sys_id_sim_requirement} holds. Hence, the 
	conclusion of the Theorem follows by way of arguments already discussed in 
	\cref{sec:a_relu_architecture_for_nonlinear_systems}.
\end{proof}

% subsection relu_architectures_for_nonlinear_system_identification (end)

% section relu_architectures_for_system_identification (end)

 %

% !TEX root = ./main.tex

\section{Conclusion} % (fold)
\label{sec:conclusion}

In this paper, we considered the problem of automatically designing assured 
ReLU NN architectures for both control and system identification of nonlinear 
control systems. To the best of our knowledge, our approach uniquely leverages 
the flexible TLL architecture and a generic specification paradigm via ADS in 
order to do \textbf{\emph{automatic, assured NN architecture design for 
incompletely known nonlinear systems}}. We thus expect this work to serve as a 
starting point for further work on the problem of automatic NN architecture 
design, particularly with closed-loop dynamical specifications in mind. In 
conclusion, we offer two observations that present opportunities for 
generalization.

First, we note that in this paper, we have constructed CPWA functions in terms 
of their linear regions (see Step 2 of the proof of \cref{thm:main_theorem}:  
\cref{sec:thm1_part2}). This construction by linear regions is ideally suited 
for conversion to TLL NNs using \cref{thm:tll_architecture}. However, in this 
paper, we have obtained these linear regions by a specific construction: first 
by setting down a regular grid of interpolation hypercubes and then subdividing 
those hypercubes into simplexes using the braid hyperplane arrangement. This 
scheme works because grid of hypercubes controls the size of the resulting 
simplexes, and the braid arrangement ensures that any full-dimensional face of 
one simplex is exactly shared by one and only one other simplex -- i.e. it 
constitutes a ``tiling'', which leads to a continuous piecewise function. 
However, this specific construction is not needed to employ 
\cref{thm:tll_architecture} to obtain a TLL network: any other ``tiling'' by 
sufficiently small simplexes could be converted to a TLL approximate controller 
just as easily. Thus, a direction of future work would be to find a more 
efficient  (or even minimal) ``tiling'' by simplexes, which would lead to 
smaller architectures.

Second, we make a similar observation but in the other direction. In 
particular, we have chosen to convert the above-described approximate CPWA 
controller specifically to a TLL NN, but a TLL may not be the most 
\emph{neuron-efficient} NN representation. In other words, there may be other 
architectures that can represent the same, approximate CPWA controller using 
considerably fewer neurons. Thus, another compelling direction of future 
research would be to seek more compact assured architectures by exploring 
classes of NN architectures other than TLL NNs.

\section{Acknowledgments} % (fold)
\label{sec:acknowledgments}
The authors wish to thank Xiaowu Sun (University of California, Irvine) for 
noting by example that a function defined on the corners of a hypercube can be  
extended by a CPWA that uses fewer regions than in Lemma 4 of 
\cite{FerlezTwoLevelLatticeNeural2020a}.
% section acknowledgments (end)

% section conclusion (end) %

\vspace{-2mm}

\bibliographystyle{ieeetr} %
\bibliography{mybib} %

%\newpage

% \clearpage %
% \begin{proofs} %
% 	\appendix

% 	\input{appendix} %
% \end{proofs}
%\section{Appendices Here}

% \newpage

\appendix

% !TEX root = ./main.tex

\section{Appendix} % (fold)
\label{sec:appendix}

% \subsection{Example of a $\delta$-$\tau$ positively invariant controller} % (fold)
% \label{sub:example_of_delta_tau_controller}

% % subsection example_of_ (end)

\subsection{Proof of \cref{lem:cpwa_interpolate_new}} % (fold)
\label{sub:proof_of_cpwa_interpolate_new}
To prove \cref{lem:cpwa_interpolate_new}, we need a few preliminary results.
\begin{proposition}[Coordinate Orderings and Braid Arrangements Regions \cite{StanleyIntroductionHyperplaneArrangements}]
	\label{prop:braid_regions}
	Let $\mathcal{B}_n$ be the braid arrangement defined in 
	\cref{lem:cpwa_interpolate_new}. Then for each (open) region $R$ of 
	$\mathcal{B}_n$ there is a unique permutation $\sigma_R : \{1, \dots, n\}  
	\rightarrow \{1, \dots, n\}$ such that
	\begin{equation}
		\forall x \in R\;.\; 
		\sigma_{R}(i) < \sigma_{R}(j) \implies \pi_{\sigma_R(i)}(x) < \pi_{\sigma_R(j)}(x).
	\end{equation}
	In other words, the coordinates of any point in $R$ can be sorted into 
	ascending order using a single, common permutation, and this permutation is 
	unique to $R$.
\end{proposition}
\begin{proof}
	See \cite{StanleyIntroductionHyperplaneArrangements}.
\end{proof}

\begin{lemma}
	\label{lem:braid_face_intersections}
	Let $\mathcal{B}_n$ be the braid arrangement, and let  
	$\mathcal{D}_{\mathcal{B}_n}$ be the braid dissection of $C_n$  (both as 
	defined in \cref{lem:cpwa_interpolate_new}).

	Then any region $R$ of $\mathcal{B}_n$ intersects exactly one 
	$n-1$-dimensional face not containing the origin, i.e.
	\begin{equation}
		\mathsf{F}_i^R \triangleq R \cap  F_{i}^1
	\end{equation}
	is nonempty for exactly one $i = i_R \in \{1, \dots, n\}$. Moreover, for 
	any two distinct regions $R$ and $R^\prime$ of $\mathcal{B}_n$,
	\begin{equation}
		\mathsf{F}_{R} \cap \mathsf{F}_{{R^\prime}} = \emptyset,
	\end{equation}
	where the notation $\mathsf{F}_{i_R}^R$ is simplified to $\mathsf{F}_R$ by 
	the above claim.

	Finally, for each $R$, the closure $\overline{\mathsf{F}_R}$ is isomorphic 
	to a region $P$ of the braid dissection $\mathcal{D}_{\mathcal{B}_{n-1}}$ 
	by means of the projection and embedding ($e_i$ is a column of the $(n 
	\times n)$ identity matrix):
	\begin{align}
		\hspace{-2.5mm} \pi_{\neg i} \hspace{-0.2mm} &: \hspace{-0.2mm} \mathbb{R}^{\scriptscriptstyle n} \rightarrow \mathbb{R}^{\scriptscriptstyle n \hspace{-0.5mm} - \hspace{-0.5mm} 1}, \; %\notag \\
		\pi_{\neg i} \hspace{-0.2mm} : \hspace{-0.2mm} x \hspace{-0.2mm} \mapsto \hspace{-0.2mm} \left( \begin{smallmatrix}
			e_1 & \dots & e_{i-1} & e_{i+1} & \dots e_n
		\end{smallmatrix}\right)^{\scriptscriptstyle \text{T}} \hspace{-0.3mm} x \\
		\hspace{-2.5mm} \mathscr{I}_i^k \hspace{-0.2mm} &: \hspace{-0.2mm} \mathbb{R}^{\scriptscriptstyle n \hspace{-0.5mm} - \hspace{-0.5mm} 1} \hspace{-1.8mm} \rightarrow \mathbb{R}^{\scriptscriptstyle n}, \; %\notag \\
		\mathscr{I}_i^k \hspace{-0.2mm} : \hspace{-0.2mm} x \hspace{-0.2mm} \mapsto \hspace{-0.2mm} (\hspace{-0.7mm} \begin{smallmatrix}
			e_1 & \dots & e_{i-1} & e_{i+1} & \dots e_n
		\end{smallmatrix} \hspace{-0.7mm} ) x + k \hspace{-0.2mm} \cdot \hspace{-0.2mm} e_i
	\end{align}
	i.e. $\pi_{\neg i}: \mathsf{F}_R \rightarrow P$ is a bijection that maps 
	corners of $C_n$ to corners of $C_{n-1}$; its inverse is $\mathscr{I}_i^1 : 
	P \rightarrow \mathsf{F}_R$.
\end{lemma}
\begin{proof}
	% Let $F_{i}^1 \in \mathcal{F}_{n}(C_{n+1})$ be the $n$-dimensional face of 
	% $C_{n+1}$ given by $F_{i}^1 \triangleq \{x \in C_{n+1} | \pi_i(x) = 1\}$ 
	% (see \cref{def:face_corner}).

	First, we claim that the any region $R$ of $\mathcal{B}_{n}$ intersects the 
	exactly one $F_i^1$. This follows because each point in (the open set) $R$ 
	necessarily has the same unique maximal coordinate by  
	\cref{prop:braid_regions}; identify this maximal coordinate as $i_R$. 
	Hence, $R$ intersects $F_{i_R}^1$, since every element $x \in 
	\text{int}(F_i^1)$ likewise has the same unique maximal coordinate of 
	$\pi_i(x) = 1$. Indeed, consider the point
	\begin{equation}
		x^\prime \hspace{-0.8mm} = \hspace{-0.6mm} ( \hspace{-0.3mm} \tfrac{\sigma_R(1)}{\sigma_{\hspace{-0.2mm}R}\hspace{-0.2mm}(1)\hspace{-0.2mm}+\hspace{-0.2mm}1} \hspace{0.3mm} \dots  \hspace{0.3mm} \tfrac{\sigma_{\hspace{-0.2mm}R}\hspace{-0.2mm}(i_R-1)}{\sigma_{\hspace{-0.2mm}R}\hspace{-0.2mm}(i_{\hspace{-0.2mm}R}\hspace{-0.2mm}-1)\hspace{-0.2mm}+\hspace{-0.2mm}1}  \hspace{0.5mm} \;  1 \;  \hspace{0.5mm}
		\tfrac{\sigma_{\hspace{-0.2mm}R}\hspace{-0.2mm}(i_{\hspace{-0.2mm}R}\hspace{-0.2mm} \hspace{-0.2mm}+\hspace{-0.2mm} 1)}{\sigma_{\hspace{-0.2mm}R}\hspace{-0.2mm}(i_{\hspace{-0.2mm}R}\hspace{-0.2mm} \hspace{-0.2mm}+\hspace{-0.2mm} 1)+1} % \\
		 \hspace{0.3mm}
		\dots  \hspace{0.3mm} \tfrac{\sigma_{\hspace{-0.2mm}R}\hspace{-0.2mm}(n)}{\sigma_{\hspace{-0.2mm}R}\hspace{-0.2mm}(n)\hspace{-0.2mm}+\hspace{-0.2mm}1}  \hspace{-0.3mm})
	\end{equation}
	which lies in: $F_{i_R}^1$ because of its $i_R^\text{th}$ coordinate; and 
	in $R$ because its coordinates respect the sorting associated with $R$. To 
	show the second claim, merely note that the regions of $\mathcal{B}_n$ are 
	disjoint by definition (they are regions of a hyperplane arrangement), so 
	$R \neq R^\prime$ implies that $F_R \cap F_{R^\prime} = \emptyset$.

	Finally, consider $\overline{\mathsf{F}_R}$ for $R$ in  $\mathcal{B}_n$. We 
	claim $\overline{\mathsf{F}_R}$ is isomorphic to the region in 
	$\mathcal{D}_{\mathcal{B}_{n-1}}$ given by $P =  \overline{R^\prime \cap 
	\text{int}(C_{n-1}) }$, where $R^\prime$ is the region of 
	$\mathcal{B}_{n-1}$ associated with:
	\begin{equation}
		\sigma_{R^\prime} : i \mapsto \begin{cases}
			\sigma_R(i) & i < i_R \\
			\sigma_R(i+1) & i \geq i_R
		\end{cases}
	\end{equation}
	which is in effect the ordering $\sigma_R$ with its maximal element 
	excluded. Thus, any point $x \in R$ maps to a point in $P$ so defined, 
	since it trivially maps to a point in $R^\prime$ via  $\pi_{\neg i}$. 
	Likewise, any point in $P \cap R^\prime$ maps to a point in $R \cap 
	F_{i_R}^1$ by the inclusion of the additional coordinate specified by 
	$\mathscr{I}_{i_R}^1$. Furthermore, these mappings extend to the closed 
	sets specified by continuity. It follows that corners of one set are mapped 
	to corners of the other and vice versa, since $\pi_{\neg i}$ and 
	$\mathscr{I}_{i_R}^1$ either remove or include a coordinate with value 1.
	% Moreover, this  intersection is isomorphic to the interior of a region from 
	%  $\mathcal{B}_{n}$ by simply dropping the $i^\text{th}$ coordinate (which 
	% is identically equal to one within $F_i^1$): the remaining coordinates are 
	% strictly ordered by the same permutation identified with $R$, and  hence 
	% correspond to a region from $\mathcal{B}_n$. Finally, this argument also 
	% shows that any two regions from $\mathcal{B}_{n+1}$ that intersect $F_i^1$ 
	% must have intersections that are isomorphic to a  \emph{distinct} regions 
	% from $\mathcal{B}_{n}$. For if they shared the same ordering of all but 
	% their $i^\text{th}$ coordinates, then they could only correspond to 
	% distinct regions if they placed their $i^\text{th}$ coordinates in 
	% different sorted positions; but this contradicts that both intersect 
	% $F_i^1$, which implies that their $i^\text{th}$ coordinates must come last 
	% in sorted order, as we argued above. 
\end{proof}

Now we can state formally state the proof of \cref{lem:cpwa_interpolate_new}.
\begin{proof}
	We claim that the dissection $\mathcal{D}_\mathcal{B}$ (see 
	\eqref{eq:dissection_definition}) consists of exactly $n!$ simplexes, each 
	of whose vertices are in $\mathcal{F}_0(C)$. Thus,  
	$\Xi_\alpha^{\mathcal{B}_n}$ has at most $n!$ linear regions, one for each  
	simplex.

	First, note that the interior of each region of $\mathcal{B}$ intersects 
	$\text{int}(C_n)$; this is because all possible coordinate orderings of the 
	point $(\tfrac{1}{2}, \tfrac{1}{3}, \dots, \tfrac{1}{n+1})$ are contained 
	in  $\text{int}(C_n)$. Hence, each region of $\mathcal{B}_n$ contributes 
	one region to $\mathcal{D}_{\mathcal{B}_n}$, so 
	$|\mathcal{D}_{\mathcal{B}_n}| = n!$.

	Thus, it remains to show that the intersection of $C_n$ with each region in 
	$\mathcal{B}$  is in fact an $n$-simplex, and we show this by induction on 
	dimension. As the base case, consider $n=2$. In particular, $\mathcal{B}_2$ 
	has only one hyperplane, and it dissects $C_2$ into two triangles (i.e. 
	$2$-simplexes), each of whose vertices are vertices of the square $C_2$. 
	Hence, suppose that $\mathcal{B}_n$ divides $C_n$ into $n$-simplexes with 
	vertices from  $\mathscr{F}_0(C_{n})$; we will show that this implies that 
	$\mathcal{B}_{n+1}$ divides $C_{n+1}$ into $n+1$-simplexes, each of whose 
	vertices are from  $\mathscr{F}_0(C_{n+1})$. 
	% To see this, first note the following: \\[-5pt]

	To complete the proof of the induction step, let $R$ be a region from 
	$\mathcal{B}_{n+1}$, and observe that for any $x \in R \cap  
	\text{int}(C_{n+1})$ and any $t > 0$, $t \cdot x$ remains in $R$, since 
	multiplication by a non-negative scalar doesn't change the ordering of the 
	coordinates of $x$. For a similar reason, there is a $t_\text{max}(x) \in 
	\mathbb{R}$ such that
	\begin{equation}
		t_\text{max}(x) \triangleq \sup_{t\cdot x \in C_{n+1}} t = \sup_{t 
		\cdot \max_i \pi_i(x)  < 1} t.
	\end{equation}
	% so that by construction
	% \begin{equation}
	% 	\begin{cases}
	% 		t\cdot x \in \text{int}(C) & 0 < t < t_\text{max}(x) \\
	% 		t \cdot x \not\in  \text{int}(C) & t \geq t_\text{max}(x)
	% 	\end{cases}.
	% \end{equation}
	Thus, we conclude that $R \cap \text{int}(C_{n+1})$ is the interior of a 
	cone that extends from the origin, and it has a (convex) base contained 
	entirely in some $F_{i_R}^1$ by  \cref{lem:braid_face_intersections}; 
	indeed, that base is isomorphic to $\overline{C_n \cap R^\prime}$ where 
	$R^\prime$ is a region from $\mathcal{B}_{n}$. Of course, by the induction 
	hypothesis,  $\overline{C_n \cap R^\prime}$ is an $n$-simplex whose 
	vertices are contained in $\mathscr{F}_0(C_n)$, so the base of the cone 
	$\overline{R \cap C_{n+1}}$ is necessarily an $n$-simplex whose vertices 
	lie in $F_{i_R}^k \cap \mathscr{F}_0(C_{n+1})$, again by 
	\cref{lem:braid_face_intersections}. Thus, those vertices, together with 
	the origin, constitute the $n+2$ vertices of $\overline{R \cap  
	\text{int}(C_{n+1})}$, and each vertex is a corner of $C_{n+1}$ by 
	construction. Hence, $\overline{R \cap  \text{int}(C_{n+1})}$ is an $n+1$ 
	simplex  whose vertices are in $\mathscr{F}_0(C_{n+1})$ as claimed.

	It remains to show that this $\Xi_\alpha^{\mathcal{B}_n}$ satisfies 
	\emph{ii)}. To wit, let $\alpha, \alpha^\prime \negthinspace :  
	\negthinspace \mathscr{F}_0(C) \negthinspace \rightarrow \negthinspace 
	\mathbb{R}$, and assume without loss of generality that $k=1$ so that
	\begin{equation}
		\forall x \in F_i^1 \cap \mathscr{F}_0(C) \;.\;
		\alpha(x) = \alpha^\prime(|x - e_i|).
	\end{equation}
	Now let $R$ be a region of $\mathcal{B}_n$ with $R \cap  F_{i_R}^1 \neq 
	\emptyset$. It suffices to show that there is a region $R^\prime$ with 
	$R^\prime \cap  \text{int}(F_i^0) \neq \emptyset$ such that $\overline{R 
	\cap F_i^1}$ and $\overline{R^\prime \cap F_i^0}$ are isomorphic under the 
	(self-inverse) translation map $\mathcal{T}_i : x \mapsto | x - e_i |$. 
	This is because $\Xi_\alpha^{\mathcal{B}_n}$ and $\Xi_{\alpha^\prime}$ 
	linearly interpolate between the values of $\alpha$ and $\alpha^\prime$ on 
	$\overline{R \cap F_i^1}$ and $\overline{R^\prime \cap F_i^0}$, 
	respectively, and  $\alpha$ and $\alpha^\prime$ agree on translates of 
	these points in their domains. For example, consider the illustration of 
	$\mathcal{D}_{\mathcal{B}_3}$ shown in \cref{fig:braid_dissection_3d}: note 
	that any triangle (2-simplex) created by $\mathcal{D}_{\mathcal{B}_3}$ on 
	one face is the single-coordinate translate of a triangle created by 
	$\mathcal{D}_{\mathcal{B}_3}$ on the opposite face.

	However, this follows as a more or less direct consequence of 
	\cref{lem:braid_face_intersections}. Indeed, we note that $\overline{R\cap  
	F_{i_R}^1}$ is isomorphic to some $P \in \mathcal{D}_{\mathcal{B}_{n-1}}$ 
	derived from a region $R^\prime_{n-1}$ in $\mathcal{B}_{n-1}$. But it is 
	easy to see that this $P$ is isomorphic to a set of the claimed form 
	$\overline{R^\prime \cap F_{i_R}^0}$, by using the inclusion map  
	$\mathscr{I}_{i_R}^0$ instead of $\mathscr{I}_{i_R}^1$ -- and with the 
	usual inverse $\pi_{\neg i_R}$ (see \cref{lem:braid_face_intersections}). 
\end{proof}

%

% !TEX root = ./main.tex

\begin{IEEEbiography}[{\includegraphics[width=1in,height=1.25in,clip,keepaspectratio]{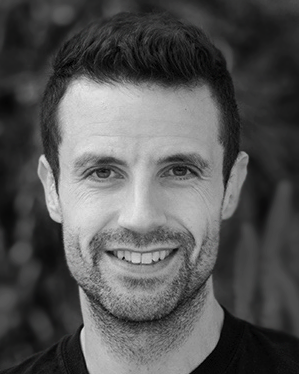}}]{James Ferlez} %
received a B.S. and M.S. in Electrical Engineering from the Pennsylvania State 
University, University Park, PA in 2008 and 2009, respectively, and earned a 
PhD in Electrical Engineering from the University of Maryland, College Park in 
2019.

He is currently a Postdoctoral Researcher at the University of California, 
Irvine in the Resilient Cyber-Physical Systems Laboratory. His research 
interests include formal methods for control, neural networks, machine learning 
and stochastic control.
\end{IEEEbiography}

\begin{IEEEbiography}[{\includegraphics[width=1in,height=1.25in,clip,keepaspectratio]{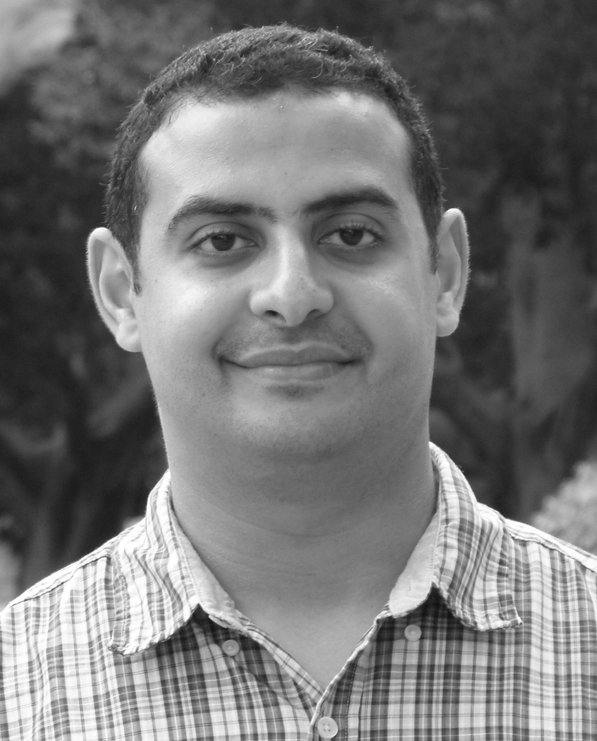}}]{Yasser Shoukry} (Member, IEEE) %
is an Assistant Professor in the Department of Electrical Engineering and Computer Science at the University of California, Irvine. He received his Ph.D. in Electrical Engineering from the University of California, Los Angeles in 2015. He received the M.Sc. and the B.Sc. degrees (with distinction and honors) in Computer and Systems engineering from Ain Shams University, Cairo, Egypt in 2010 and 2007, respectively. Between September 2015 and July 2017, Yasser was a joint post-doctoral associate at UC Berkeley, UCLA, and UPenn. Before pursuing his Ph.D. at UCLA, he spent four years as an R\&D engineer in the industry of automotive embedded systems. Yasser’s research interests include the design and implementation of resilient Cyber–Physical Systems (CPS) and Internet-of-Things (IoT) by drawing on tools from embedded systems, formal methods, control theory, and machine learning.

Prof. Shoukry is the recipient of the IEEE TC-CPS Early Career Award in 2021, the NSF CAREER award in 2019, the Best Demo Award from the ACM/IEEE IPSN conference in 2017, the Best Paper Award from the ACM/IEEE ICCPS in 2016 and the Distinguished Dissertation Award from UCLA EE department in 2016. In 2015, he led the UCLA/Caltech/CMU team to win the NSF Early Career Investigators (NSF-ECI) research challenge. His team represented the NSF-ECI in the NIST Global Cities Technology Challenge, an initiative designed to advance the deployment of Internet of Things (IoT) technologies within a smart city. He is also the recipient of the 2019 George Corcoran Memorial Award for his contributions to teaching and educational leadership in the field of CPS and IoT.

\end{IEEEbiography}

\end{document}